\documentclass{aptpub}

\authornames{Claire Launay, Bruno Galerne, Agn\`es Desolneux} 
\shorttitle{Exact Sampling of DPPs without Eigendecomposition} 

\newcommand{\C}{\mathbb{C}} 
\newcommand{\R}{\mathbb{R}} 
\renewcommand{\P}{\mathbb{P}} 

\newcommand{\Y}{\mathcal{Y}} 

\newcommand{\DPP}{\operatorname{DPP}} 
\newcommand{\tr}{\operatorname{tr}} 
\newcommand{\Tr}{\operatorname{Tr}} 

\usepackage{amsmath}
\usepackage{amssymb}
\usepackage{graphicx}
\usepackage{microtype}
\usepackage{ulem}
\usepackage{color}
\usepackage{booktabs} 
\usepackage{amsmath}
\usepackage{enumitem}
\usepackage{float}
\usepackage{algorithm}
\usepackage{algorithmic}
\usepackage{subcaption}
\usepackage[hidelinks]{hyperref}
\usepackage{enumitem}
\usepackage{ulem}

\begin{document}

\normalem
                               
\title{Exact Sampling of Determinantal Point Processes without Eigendecomposition} 

\authorone[Universit\'e de Paris]{Claire Launay} 

\addressone{Laboratoire MAP5\\
       Universit\'e de Paris, CNRS\\
       Paris, 75006, FRANCE} 
       
\authortwo[Universit\'e d'Orl\'eans]{Bruno Galerne} 

\addresstwo{Institut Denis Poisson,\\
       Universit\'e d'Orl\'eans, Universit\'e de Tours, CNRS\\
       Orl\'eans, 45100, FRANCE}       
       
\authorthree[CNRS and ENS Paris-Saclay]{Agn\`es Desolneux} 

\addressthree{Centre Borelli, CNRS\\
       ENS Paris Saclay\\
       Gif-sur-Yvette, 91190, FRANCE}
       
\begin{abstract}
Determinantal point processes (DPPs) enable the modeling of repulsion: they provide diverse sets of points. The repulsion is encoded in a kernel $K$ that can be seen, in a discrete setting, as a matrix storing the similarity between points. The main exact algorithm to sample DPPs uses the spectral decomposition of $K$, a computation that becomes costly when dealing with a high number of points. Here, we present an alternative exact algorithm to sample in discrete spaces that avoids the eigenvalues and the eigenvectors computation. The method used here is innovative and numerical experiments show competitive results with respect to the initial algorithm.
\end{abstract}

\keywords{Determinantal point processes; Exact Sampling; Thinning; Cholesky decomposition; General marginal} 

\ams{68U20}{60G55} 


Determinantal point processes (DPPs) are processes that capture negative correlations. The more similar two points are, the less likely they are to be sampled simultaneously. Then DPPs tend to create sets of diverse points. They naturally arise in random matrix theory \cite{Ginibre_1965} or in the modelling of a natural repulsive phenomenon like the repartition of trees in a forest \cite{Lavancier_Moller_Rubak_dpp_and_statistical_inference_2015}. Ever since the work of Kulesza and Taskar \cite{Kulesza_Taskar_dpp_for_machine_learning_2012}, these processes have become more and more popular in machine learning, because of their ability to draw subsamples that account for the inner diversity of data sets. This property is useful for many applications, such as summarizing documents \cite{Dupuy_Bach_learning_dpp_sublinear_time_arXiv2016},  improving a stochastic gradient descent by drawing diverse subsamples at each step \cite{Zhang_Minibatch_and_DPP_UAI2017} or extracting a meaningful subset of a large data set to estimate a cost function or some parameters \cite{Tremblayetal_DPPforCoresets_2018,Bardenet_StatApplicationsofDPPs_2017,Amblard_Subsampling_with_kDPP_2018}.
Several issues are under study, as learning DPPs, for instance through maximum likelihood estimation \cite{Kulesza_Taskar_learningDPP_2012,Brunel_RatesofEstimationforDPP_2017}, or sampling these processes. Here we will focus on the sampling question and we will only deal with a discrete and finite determinantal point process $Y$, defined by its kernel matrix $K$, a configuration particularly adapted to machine learning data sets. 

The main algorithm to sample DPPs is a spectral algorithm \cite{Hough_et_al_dpp_and_independence_2006}: it uses the eigendecomposition of $K$ to sample $Y$. It is exact and in general quite fast. Yet, the computation of the eigenvalues of $K$ may be very costly when dealing with large-scale data. That is why numerous algorithms have been conceived to bypass this issue.
Some authors tried to design a sampling algorithm adapted to specific DPPs. For instance, it is possible to speed up the initial algorithm by assuming that $K$ has a bounded rank \cite{Kulesza_Taskar_StructuredDPP_2010_NIPS,Gartrell_Low_rank_factorization_of_DPP_2016}. These authors use a dual representation of the kernel so that almost all the computations in the spectral algorithm are reduced. One can also deal with another class of DPPs associated to kernels $K$ that can be decomposed in a sum of tractable matrices \cite{Dupuy_Bach_learning_dpp_sublinear_time_arXiv2016}. In this case, the sampling is much faster and the authors study the inference on these classes of DPPs. At last, Propp and Wilson \cite{Propp_exact_sampling_spanning_tree_1998} use Markov chains and the theory of coupling from the past to sample exactly particular DPPs: uniform spanning trees. Adapting Wilson's algorithm, Avena and Gaudilli\`ere \cite{Avena_Randomspanningforests_2018} provide another algorithm to efficiently sample a parametrized DPP kernel associated to random spanning forests.

Another type of sampling algorithms is the class of approximate methods. Some authors approach the original DPP with a low rank matrix, either by random projections \cite{Kulesza_Taskar_dpp_for_machine_learning_2012,Gillenwater_Document_collection_2012} or using the Nystrom approximation \cite{AffandiKulesza_NystromApprox2013}.
The Monte Carlo Markov Chain methods offer also nice approximate sampling algorithms for DPPs. It is possible to obtain satisfying convergence guarantees for particular DPPs; for instance, k-DPPs with fixed cardinality \cite{AnariGR16_MCMCforStronglyRayleighDistributionsandDPP,Li_Jegelka_Sra_efficient_sampling_kdpp_AISTATS2016} or projection DPPs \cite{Gautier_Bardenet_Valko_zonotope_hit_and_run_ICML2017}. Li et al. \cite{Fast_mixing_mc_strongly_rayleign_measures_NIPS2016} even proposed a polynomial-time sampling algorithm for general DPPs, thus correcting the initial work of Kang \cite{Kang_fast_dpp_sampling_clustering_NIPS2013}. These algorithms are commonly used as they save significant time but the price to pay is the lack of precision of the result.

As one can see, except the initial spectral algorithm, no algorithm allows for the exact sampling of a general DPP.
The main contribution of this paper is to introduce such a general and exact algorithm, that does not involve the kernel eigendecomposition, to sample discrete DPPs.
The proposed algorithm is a sequential thinning procedure that relies on two new results: (i) the explicit formulation of the marginals of any determinantal point process and (ii) the derivation of an adapted Bernoulli point process containing a given DPP. This algorithm was first presented in \cite{LaunayGalerneDesolneux_exactDPPsampling_2018} and was, to our knowledge, the first exact sampling strategy without spectral decomposition. \textsc{Matlab} and Python implementations of this algorithm (using the PyTorch library in the Python code) are available online (\url{https://www.math-info.univ-paris5.fr/~claunay/exact_sampling.html}) and hopefully soon in the
repository created by Guillaume Gautier \cite{Gautier_DPPy_2018} gathering exact and approximate DPP sampling algorithms.
Let us mention that three very recent preprints \cite{Poulson_exactDPPsampling_2019,Gillenwater_TreebasedDPPsampling_2019,Derezinskietal_DPP_VFX_2019} also propose new algorithms to sample general DPPs without spectral decomposition. Poulson
\cite{Poulson_exactDPPsampling_2019} presents factorization strategies of Hermitian and non-Hermitian DPP kernels to sample general determinantal point processes.
As our algorithm, it heavily relies on Cholesky decomposition. Gillenwater and al. \cite{Gillenwater_TreebasedDPPsampling_2019} use the dual representation of $L$-ensembles presented in \cite{Kulesza_Taskar_dpp_for_machine_learning_2012} to construct a binary tree containing enough information on the kernel to sample DPPs in sublinear time. Derezi\'nski et al. \cite{Derezinskietal_DPP_VFX_2019} apply a preprocessing step that preselects a portion of the points using a regularized DPP. Then, a usual DPP sampling is done on the selection. 
This is related to our thinning procedure of the initial set by a Bernoulli point process. 
However note that the authors report that the overall complexity of their sampling scheme is sublinear while ours is cubic due to Cholesky decomposition. Finally, in \cite{Blasz_DetermThinning_2019}, Blaszczyszyn and Keeler present a similar procedure based on a continuous space: they use discrete determinantal point processes to thin a Poisson point process defined on that continuous space. The point process generated offers theoretical guarantees on repulsion and is applied to fit network patterns.

The rest of the paper is organized as follows: in the next section, we present the general framework of determinantal point processes and the classic spectral algorithm. In Section 2, we provide an explicit formulation of the general marginals and pointwise conditional probabilities of any determinantal point process, from its kernel $K$. Using these formulations, we first introduce a ``naive", exact but slow, sequential algorithm that relies on the Cholesky decomposition of the kernel $K$. In Section 3, using the thinning theory, we accelerate the previous algorithm and 
introduce a new exact sampling algorithm for DPPs that we call the sequential thinning algorithm. Its computational complexity is compared with that of the two previous algorithms. In Section 4, we display the results of some experiments comparing these three sampling algorithms and we describe the conditions under which the sequential thinning algorithm is more efficient than the spectral algorithm. Finally, we discuss and conclude on this algorithm.

\section{DPPs and their Usual Sampling Method: the Spectral Algorithm}

In the next sections, we will use the following notations. Let us consider a discrete finite set $\mathcal{Y} = \{1, \dots, N\}$. This set represents the space on which the point process is defined. In point process theory, it can be called the carrier space or state space. In this paper, we choose a machine learning term and refer to $\Y$ as the ground set. 
For $M \in \R^{N\times N}$ a matrix, we will denote by $M_{A \times B}$, $\forall A,B \subset \Y$,  the matrix $\left(M(i,j)\right)_{(i,j) \in A \times B}$ and the short notation $M_A = M_{A \times A}$.
Suppose that $K$ is a Hermitian positive semi-definite matrix of size $N \times N$, indexed by the elements of $\mathcal{Y}$, so that any of its eigenvalues is in $[0,1]$. A subset $Y \subset \mathcal{Y}$ is said to follow a DPP distribution of kernel $K$ if, 
$$ \mathbb{P}\left( A \subset Y \right) = \det(K_A), \; \forall \, A \subset \mathcal{Y}.$$

The spectral algorithm is standard for drawing a determinantal point process. It relies on the eigendecompostition of its kernel $K$. It was first introduced by Hough et al. \cite{Hough_et_al_dpp_and_independence_2006} and is also presented in a more detailed way by Scardicchio \cite{Scardicchio_Statistical_properties_of_DPP_2009}, Kulesza and Taskar \cite{Kulesza_Taskar_dpp_for_machine_learning_2012} or Lavancier et al. \cite{Lavancier_Moller_Rubak_dpp_and_statistical_inference_2015}. It proceeds in 3 steps: the first step is the computation of the eigenvalues $\lambda_j$ and the eigenvectors $v^j$ of the matrix $K$. The second step consists in randomly selecting a set of eigenvectors according to $N$ Bernoulli variables of parameter $\lambda_i$, for $i=1,\dots,N$. The third step is drawing sequentially the associated points using a Gram-Schmidt process.

\begin{algorithm}
\begin{enumerate}[itemsep=0mm]
\item Compute the orthonormal eigendecomposition $(\lambda_j, v^j)$ of the matrix $K$.
\item Select a random set of eigenvectors: Draw a Bernoulli process $\mathbf{X} \in \{0,1\}^N$ with parameter $(\lambda_j)_{j}$.
Denote by $n$ the number Bernoulli samples equal to one, $\displaystyle{\{\mathbf{X}=1\} = \{j_1,\dots, j_n\}}$. Define the matrix
$ V = \left( v^{j_1} \, v^{j_2} \, \cdots \, v^{j_n} \right) \in \R^{N\times n} $
and denote by $V_k \in\R^n$ the $k$-th line of $V$, for $k\in\Y$.
\item Return the sequence $Y = \{y_1,y_2,\dots,y_n\}$ sequentially drawn as follows: \\
For $l=1$ to $n$
\begin{itemize}
\item Sample a point $y_l\in\Y$ from the discrete distribution, 
$$\hspace{-1cm}
p^l_k = \frac{1}{n-l+1}\left( \|V_k\|^2 - \sum_{m=1}^{l-1} |  \langle V_{k}, e_m\rangle |^2 \right),\forall k \in \Y.
$$
\item If $l<n$, define $e_l = \frac{w_l}{\|w_l\|} \in\R^n$ where 
$w_l = V_{y_{l}} - \sum_{m=1}^{l-1} \langle V_{y_l}, e_m\rangle e_m.$
\end{itemize}
\end{enumerate}

\caption{The spectral sampling algorithm}
\label{alg_spectral_simulation_dpp}
\end{algorithm}

This algorithm is exact and relatively fast but it becomes slow when the size of the ground set grows. For a ground set of size $N$ and a sample of size $n$, the third step costs $O(Nn^3)$ because of the Gram-Schmidt orthonormalisation. Tremblay et al. \cite{Tremblay_AlgoforDPP_2018} propose to speed it up using optimized computations and they achieve the complexity $O(Nn^2)$ for this third step. Nevertheless, the eigendecomposition of the matrix $K$ is the heaviest part of the algorithm, as it runs in time $O(N^3)$, and we will see in the numerical results that this first step represents in general more than $90\%$ of the running time of the spectral algorithm. As nowadays the amount of data explodes, in practice the matrix $K$ is very large so it seems relevant to try to avoid this costly operation. We compare the time complexities of the different algorithms presented in this paper at the end of Section \ref{sec:thinning}. In the next section, we show that any DPP can be exactly sampled by a sequential algorithm that does not require the eigendecomposition of $K$.

\section{Sequential Sampling Algorithm}

Our goal is to build a competitive algorithm to sample DPPs that does not involve the eigendecomposition of the matrix $K$. To do so, we first develop a ``naive" sequential sampling algorithm and subsequently, we will accelerate it using a thinning procedure, presented in Section \ref{sec:thinning}.

\subsection{Explicit General Marginal of a DPP}

First, we need to specify the marginals and the conditional probabilities of any DPP. When $I-K$ is invertible, a formulation of the explicit marginals already exists \cite{Kulesza_Taskar_dpp_for_machine_learning_2012}, it implies to deal with a $L$-ensemble matrix $L$ instead of the matrix $K$. However, this hypothesis is reductive: among others, it ignores the useful case of projection DPPs, when the eigenvalues of $K$ are either 0 or 1. We show below that general marginals can easily be formulated from the  associated kernel matrix $K$. For all $A \subset \Y$, we denote $I^A$ the $N \times N$ matrix with $1$ on its diagonal coefficients indexed by the elements of $A$, and $0$ anywhere else. We also denote $|A|$ the cardinality of any subset $A \subset \Y$ and $\overline{A} \in \mathcal{Y}$ the complementary set of $A$ in $\mathcal{Y}$.

\begin{proposition}[Distribution of a DPP]
For any $A\subset\Y$, we have
$$
\P(Y = A) = (-1)^{|A|} \det(I^{\overline{A}} - K).
$$
\end{proposition}

\begin{proof}
We have that
$\displaystyle \P(A\subset Y) = \sum_{B\supset A} \P(Y = B).$
Using the M\"{o}bius inversion formula (see Appendix \ref{app:mobius}), for all $A\subset\Y, $
$$
\begin{aligned} \P(Y = A) &= \sum_{B\supset A} (-1)^{|B\setminus A|} \P(B \subset Y) = (-1)^{|A|} \sum_{B\supset A} (-1)^{|B|} \det(K_B)
\\
& = (-1)^{|A|} \sum_{B\supset A} \det((-K)_B)
\end{aligned}
$$
Furthermore, Kulesza and Taskar \cite{Kulesza_Taskar_dpp_for_machine_learning_2012} state in Theorem 2.1 that for all $\displaystyle L \in \R^{N \times N},$ for all $ \displaystyle
 A \subset~\mathcal{Y}, \sum_{A \subset B \subset \mathcal{Y}}
 \det(L_B) = \det(I^{\overline{A}}+L)$. Then we obtain
$$ \P(Y = A) = (-1)^{|A|} \det(I^{\overline{A}} - K).$$

\end{proof}

We have by definition $\P(A \subset Y) = \det(K_A)$ for all $A$, and as a consequence $\P(B~\cap~Y = \emptyset) = \det((I - K)_B)$ for all $B$. The next proposition gives for any DPP the expression of the general marginal $\P(A \subset Y, B \cap Y = \emptyset)$, for any $A,B$ disjoint subsets of $\mathcal{Y}$, using $K$. In what follows, $H^B$ denotes the symmetric positive semi-definite matrix 
$$H^B = K + K_{\Y \times B} ((I-K)_B)^{-1} K_{B\times \Y}.$$

\begin{theorem}
[General Marginal of a DPP]
Let $A, B\subset\Y$ be disjoint.
If $\P(B\cap Y = \emptyset) = \det((I-K)_B) = 0$, then
$\displaystyle{\P(A\subset Y, B\cap Y = \emptyset)=0}$.
Otherwise, the matrix $(I-K)_B$ is invertible and
$$
\P(A\subset Y, B\cap Y = \emptyset) = \det((I-K)_B)\det(H^B_A).
$$
\end{theorem}

\begin{proof}
Let $A, B\subset\Y$ disjoint such that $\P(B\cap Y = \emptyset)  \neq 0$.
Using the previous proposition,
$$
\P(A\subset Y, B\cap Y = \emptyset)= \sum_{A\subset C \subset \overline{B}} \P(Y = C)
= \sum_{A\subset C \subset \overline{B}} (-1)^{|C|} \det(I^{\overline{C}} - K).
$$
For any $C$ such that $A\subset C \subset \overline{B}$, one has $B\subset \overline{C}$. Hence, by reordering the matrix coefficients, and using the Schur's determinant formula \cite{Horn_MatrixAnalysis_1990},
$$
\begin{aligned}
\det(I^{\overline{C}} - K)
&= \det
\begin{pmatrix}
(I^{\overline{C}} - K)_B  &  (I^{\overline{C}} - K)_{B\times \overline{B}} \\
(I^{\overline{C}} - K)_{\overline{B}\times B} & (I^{\overline{C}} - K)_{\overline{B}}
\end{pmatrix}\\
&= \det
\begin{pmatrix}
(I - K)_B  &  -K_{B\times \overline{B}} \\
-K_{\overline{B}\times B} & (I^{\overline{C}} - K)_{\overline{B}}
\end{pmatrix}\\
&= \det((I - K)_B) \det(( I^{\overline{C}} - H^B)_{\overline{B}}).
\end{aligned}
$$
Thus, $\displaystyle \P(A\subset Y, B\cap Y = \emptyset)= \det((I - K)_B) \sum_{A\subset C \subset \overline{B}} (-1)^{|C|} \det( ( I^{\overline{C}} -H^B)_{\overline{B}}).$\\
According to Theorem 2.1 in Kulesza and Taskar~\cite{Kulesza_Taskar_dpp_for_machine_learning_2012}, for all $A\subset\overline{B}$,
$$
\quad \sum_{A\subset C \subset \overline{B}} \det( - H^B_C) = \det( ( I^{\overline{A}} -H^B)_{\overline{B}}).
$$
Then, M\"{o}bius inversion formula ensures that, $\forall A\subset\overline{B},$
$$
\hspace{-0.1cm}\sum_{A\subset C \subset \overline{B}} (-1)^{|C\setminus A|} \det( ( I^{\overline{C}} -H^B)_{\overline{B}})
= \det( - H^B_A)
= (-1)^{|A|} \det(H^B_A).
$$
Hence, 
$\P(A\subset Y, B\cap Y = \emptyset) 
= \det((I - K)_B) \det(H^B_A)$.
\end{proof}

With this formula, we can explicitly formulate the pointwise conditional probabilities of any DPP.

\begin{corollary}[Pointwise conditional probabilities of a DPP]
\label{prop_pointwise_conditional_proba_dpp}
Let $A, B\subset\Y$ be two disjoint sets such that $\displaystyle{\P(A\subset Y,~B\cap Y = \emptyset)\neq 0}$, and let $k\notin A\cup B$.
Then,
\begin{equation}
\begin{aligned}
\P(\{k\}\subset Y | A \subset Y,~B\cap Y = \emptyset)
&= \frac{\det(H^B_{A\cup\{k\}})}{\det(H^B_A)}\\
&= H^B(k,k) - H^B_{\{k\}\times A} (H^B_A)^{-1} H^B_{A\times \{k\}}.
\label{eq_pointwise_conditional_proba_dpp}
\end{aligned}
\end{equation}
\end{corollary}

This is a straightforward application of the previous expression and the Schur determinant formula \cite{Horn_MatrixAnalysis_1990}. Note that these pointwise conditional probabilities are related to the Palm distribution of a point process \cite{Chiu_stochasticGeom_2013} which characterizes the distribution of the point process under the condition that there is a point at some location $x \in \Y$. Shirai and Takahashi proved in \cite{Shirai2003I} that DPPs on general spaces are closed under Palm distributions, in the sense that there exists a DPP kernel $K^x$ such that the Palm measure associated to DPP($K$) and $x$ is a DPP defined on $\Y$ with kernel $K^x$. Borodin and Rains \cite{Borodin_Schurprocess_2005} also provide similar results on discrete spaces, using $L$-ensembles, that Kulesza and Taskar adapt in \cite{Kulesza_Taskar_dpp_for_machine_learning_2012}. They condition the DPP not only on a subset included in the point process but also, similarly as Corollary \ref{eq_pointwise_conditional_proba_dpp}, on a subset not included in the point process. As Shirai and Takahashi, they derive a formulation of the generated marginal kernel $L$.\\
\indent Now, we have all the necessary expressions for the sequential sampling of a DPP.

\subsection{Sequential Sampling Algorithm of a DPP}

This sequential sampling algorithm simply consists in using Formula \eqref{eq_pointwise_conditional_proba_dpp} and updating at each step the pointwise conditional probability, knowing the previous selected points.  It is presented in Algorithm \ref{algo_sequential_sampling_dpp}. We recall that this sequential algorithm is the first step toward developing a competitive sampling algorithm for DPPs: with this method, one doesn't need eigendecomposition anymore. The second step (Section \ref{sec:thinning}) will be to reduce its computational cost.

\begin{algorithm}[H]
\begin{itemize}
\item Initialization: $A \leftarrow \emptyset$, $B \leftarrow \emptyset$.
\item For $k=1$ to $N$:
\begin{enumerate}[leftmargin=-0.1cm]
\item Compute $H^B_{A\cup\{k\}} =\displaystyle{ K_{A\cup\{k\}} + K_{A\cup\{k\} \times B} ((I-K)_B)^{-1} K_{B\times A\cup\{k\}}}$.
\item Compute the probability $p_k$ given by
$$
p_k 
= \P\left(\{k\}\subset Y | A \subset Y,~B\cap Y = \emptyset\right)
= H^B(k,k) - H^B_{\{k\}\times A} (H^B_A)^{-1} H^B_{A\times \{k\}}.
$$
\item With probability $p_k$, $k$ is included, $A\leftarrow A\cup\{k\}$, otherwise $\displaystyle{B\leftarrow B\cup\{k\}}$.
\end{enumerate}
\item Return $A$.
\end{itemize}
\caption{Sequential sampling of a DPP with kernel $K$}
\label{algo_sequential_sampling_dpp}
\end{algorithm}

The main operations of Algorithm~\ref{algo_sequential_sampling_dpp} involve solving linear systems related to $(I-K)_B^{-1}$. Fortunately, here we can use the Cholesky factorization, which alleviates the computational cost.
Suppose that $T^B$ is the Cholesky factorization of $(I-K)_B$, that is,
$T^B$ is a lower triangular matrix such that $(I-K)_B = T^B (T^{B})^{*}$ (where $(T^{B})^{*}$ is the conjugate transpose of $T^{B}$).
Then, denoting $J^B = (T^{B})^{-1} K_{B\times A\cup\{k\}}$, one simply has 
$
H^B_{A\cup\{k\}} = \displaystyle K_{A\cup\{k\}} + (J^{B})^{*} J^B.
$

Furthermore, at each iteration where $B$ grows, the Cholesky decomposition $T^{B\cup\{k\}}$ of $(I-K)_{B\cup\{k\}}$ can be computed from $T^B$ using standard Cholesky update operations, involving the resolution of only one linear system of size $|B|$. See Appendix \ref{sec:cholesky_update} for the details of a typical Cholesky decomposition update.

In comparison with the spectral sampling algorithm of Hough et al.~\cite{Hough_et_al_dpp_and_independence_2006}, one requires computations for each site of $\Y$, and not just one for each sampled point of $Y$. We will see at the end of Section \ref{sec:thinning} and in the experiments that it is not competitive.

\section{Sequential Thinning Algorithm} \label{sec:thinning}

In this section, we show that we can significantly decrease the number of steps and the running time of Algorithm \ref{algo_sequential_sampling_dpp}: we propose to first sample a point process $X$ containing $Y$, the desired DPP, and then make a sequential selection of the points of $X$ to obtain $Y$.
This procedure can be called a sequential thinning.

\subsection{General Framework of Sequential Thinning}

We first describe a general sufficient condition for which a target point process $Y$ - it will be a determinantal point process in our case -  can be obtained as a sequential thinning of a point process $X$. This is a discrete adaptation of the thinning procedure on the continuous line of Rolski and Szekli~\cite{Rolski_Szekli_stochastic_ordering_and_thinning_of_point_processes_1991}.
To do this, we will consider a coupling $(X,Z)$ such that $Z\subset X$ will be a random selection of the points of $X$ and that will have the same distribution as $Y$. From this point onward, we identify the set $X$ with the vector of size $N$ with $1$ in the place of the elements of $X$ and $0$ elsewhere, and we use the notations $X_{1:k}$ to denote the vector $(X_1,\dots,X_k)$ and $0_{1:k}$ to denote the null vector of size $k$. We want to define the random vector $(X_1,Z_1,X_2,Z_2,\dots,X_N,Z_N)\in\R^{2N}$ with the following conditional distributions for $X_k$ and $Z_k$:
\begin{equation}
\begin{cases}
\P(X_k = 1|Z_{1:k-1} = z_{1:k-1}, X_{1:k-1} = x_{1:k-1})  = \P(X_k = 1|X_{1:k-1} = x_{1:k-1}) \vspace{0.5cm}\\

\P(Z_k = 1|Z_{1:k-1} = z_{1:k-1}, X_{1:k} = x_{1:k}) = \mathbf{1}_{\{x_k=1\}} \dfrac{\P(Y_k = 1|Y_{1:k-1} = z_{1:k-1})}{\P(X_k = 1| X_{1:k-1}=x_{1:k-1})}.
\end{cases}
\label{eq:definitions_sequential_thinning}
\end{equation}

\begin{proposition}[Sequential thinning]
\label{prop_sequential_thinning}
Assume that $X,Y, Z$ are discrete point processes on $\Y$ that satisfy 
for all $k\in\{1,\dots,N\}$, and all $z$, $x\in\{0,1\}^N$,
\begin{equation}
\begin{array}{c}
\P(Z_{1:k-1}=z_{1:k-1},X_{1:k-1}=x_{1:k-1}) > 0\\
\text{implies}\\
\P(Y_{k} = 1|Y_{1:k-1} = z_{1:k-1}) \leq \P(X_{k} = 1| X_{1:k-1}=x_{1:k-1}).
\end{array}
\label{eq:necessary_and_sufficient_condition_for_thinning}
\end{equation}
Then, it is possible to choose $(X,Z)$ in such a way that \eqref{eq:definitions_sequential_thinning} is satisfied. In that case, we have that $Z$ is a thinning of $X$, that is $Z\subset X$, and $Z$ has the same distribution as $Y$.
\end{proposition}

\begin{proof}
Let us first discuss the definition of the coupling $(X,Z)$.
With the conditions~\eqref{eq:necessary_and_sufficient_condition_for_thinning},
the ratios defining the conditional probabilities of Equation~\eqref{eq:definitions_sequential_thinning} are ensured to be between $0$ and $1$ 
(if the conditional events have non zero probabilities).
Hence the conditional probabilities allows us to construct sequentially the distribution of the random vector 
$(X_1,Z_1,X_2,Z_2,\dots,X_N,Z_N)$ of length $2N$, and thus the coupling is well-defined.
Furthermore, as Equation \eqref{eq:definitions_sequential_thinning} is satisfied, $Z_k = 1$ only if $X_k =1$, so one has $Z\subset X$.

Let us now show that $Z$ has the same distribution as $Y$. By complementarity of the events $\{Z_k = 0\}$ and $\{Z_k = 1\}$, it is enough to show that for all $k\in\{1,\dots,N\}$, and $z_1,\dots,z_{k-1}$ such that $\P(Z_{1:k-1} = z_{1:k-1})>0$,
\begin{equation}
 \P(Z_k = 1|Z_{1:k-1} = z_{1:k-1}) = \P(Y_k = 1|Y_{1:k-1} = z_{1:k-1}).
\label{eq:same_conditional_probability}
\end{equation}

Let $k\in\{1,\dots,N\}$, $(z_{1:k-1},x_{1:k-1}) \in\{0,1\}^{2(k-1)}$, such that $
\P(Z_{1:k-1}=z_{1:k-1},X_{1:k-1}=x_{1:k-1}) > 0.
$
Since $Z \subset X$, $\{Z_k = 1\} = \{Z_k = 1, X_k = 1\}$.
Suppose first that $\P(X_k = 1|X_1 = x_1,\dots, X_{k-1} = x_{k-1})\neq 0$. Then
$$
\begin{aligned}
& \P(Z_k = 1|Z_{1:k-1} = z_{1:k-1},X_{1:k-1}=x_{1:k-1})
\\
&=  \hspace{-0.05cm}\P(Z_k = 1,X_k = 1|Z_{1:k-1} = z_{1:k-1}, X_{1:k-1}=x_{1:k-1})
\\
& = \hspace{-0.22cm}
\begin{array}{l}
\P(Z_k = 1|Z_{1:k-1} = z_{1:k-1},X_{1:k-1}=x_{1:k-1},X_k = 1)
\\
\times \P(X_k = 1|Z_{1:k-1} = z_{1:k-1}, X_{1:k-1}=x_{1:k-1})
\end{array}
\\
& = \P(Y_k = 1|Y_{1:k-1} = z_{1:k-1}), \text{ by Equations \eqref{eq:definitions_sequential_thinning}}.
\end{aligned}
$$
If $\P(X_k = 1|X_{1:k-1}=x_{1:k-1}) = 0$, then
$\displaystyle
\P(Z_k = 1|Z_{1:k-1} = z_{1:k-1},X_{1:k-1}=x_{1:k-1}) = 0
$ and using~\eqref{eq:necessary_and_sufficient_condition_for_thinning}, $\P(Y_k = 1|Y_{1:k} = z_{1:k}) =0$.
Hence the identity
$$ \P(Z_k = 1|Z_{1:k-1} = z_{1:k-1},X_{1:k-1}=x_{1:k-1})
 = \P(Y_k = 1|Y_{1:k-1} = z_{1:k-1})$$
is always valid. 
Since the values $x_1,\dots,x_{k-1}$ do not influence this conditional probability, one can conclude that given $(Z_1,\dots,Z_{k-1})$, 
$Z_k$ is independent of $X_1,\dots,X_{k-1}$, and thus \eqref{eq:same_conditional_probability} is true.
\end{proof}

The characterization of the thinning defined here allows both extreme cases: there can be no pre-selection of points by $X$, meaning that $X = \Y$ and that the DPP $Y$ is sampled by Algorithm \ref{algo_sequential_sampling_dpp}, or there can be no thinning at all, meaning that the final process $Y$ can be equal to the dominating process $X$. Regarding sampling acceleration, a good dominating process $X$ must be sampled quickly and with a cardinality as close as possible to $|Y|$.

\subsection{Sequential Thinning Algorithm for DPPs}

In this section, we use the sequential thinning approach, where $Y$ is a DPP of kernel $K$ on the ground set $\Y$, and $X$ is a Bernoulli point process (BPP). BPPs are the fastest and easiest point processes to sample. $X$ is a Bernoulli process if the components of the vector $(X_1,\dots, X_N)$ are independent. Its distribution is determined by the probability of occurrence of each point $k$, that we denote by $q_k = \P(X_k=1)$. Due to the independence property,  the conditions~\eqref{eq:necessary_and_sufficient_condition_for_thinning} simplifies to
\begin{equation*}
\begin{array}{c}
\P(Z_{1:k-1}=z_{1:k-1},X_{1:k-1}=x_{1:k-1}) > 0\\
\text{implies}\\
\P(Y_k = 1|Y_{1:k-1} = z_{1:k-1}) \leq q_k.
\end{array}
\end{equation*}

The second inequality does not depend on $x$, hence it must be valid as soon as there exists a vector $x$ such that
$\P(Z_{1:k-1}=z_{1:k-1},X_{1:k-1}=x_{1:k-1}) > 0$, that is, as soon as $\P(Z_{1:k-1}=z_{1:k-1})>0$.
Since we want $Z$ to have the same distribution as $Y$, we finally obtain the conditions
$$
\forall y\in\{0,1\}^N,\; \P(Y_{1:k-1}=y_{1:k-1})>0 ~ \text{implies }
\P(Y_k = 1|Y_{1:k-1}=y_{1:k-1}) \leq q_k.
$$

Ideally, we want the $q_k$ to be as small as possible to ensure that the cardinality of $X$ is as small as possible. So
we look for the optimal values $q_k^*$, that is,
$$
q_k^* 
= \hspace{-0.3cm}\max_{
\begin{array}{c}
\scriptstyle {(y_{1:k-1})\;\in\;\{0,1\}^{k-1}\;\text{s.t.}}\\ 
\scriptstyle{\P(Y_{1:k-1}\; =\; y_{1:k-1})\;>\;0}
\end{array}}
\hspace{-0.3cm}\P(Y_k = 1|Y_{1:k-1} = y_{1:k-1}).
$$
A priori, computing $q_k^*$ would raise combinatorial issues. However, due to the repulsive nature of DPPs, we have the following proposition.

\begin{proposition}
Let $A, B\subset\Y$ be two disjoint sets such that $\P(A\subset Y,~B\cap Y = \emptyset)\neq 0$, and let $k \neq l \in \overline{A\cup B}$.
If $\P(A\cup\{l\} \subset Y,~B\cap Y = \emptyset) >0$, then
$$
\P(\{k\}\subset Y | A\cup\{l\} \subset Y,~B\cap Y = \emptyset)\leq \P(\{k\}\subset Y | A \subset Y,~B\cap Y = \emptyset).
$$
If \, $\P(A \subset Y,~(B\cup\{l\})\cap Y = \emptyset) >0$, then
$$
\P(\{k\}\subset Y | A \subset Y,~(B\cup\{l\})\cap Y = \emptyset)\geq \P(\{k\}\subset Y | A \subset Y,~B\cap Y = \emptyset).
$$
Consequently, for all $k\in\Y$, if $y_{1:k-1} \leq z_{1:k-1}$ (where $\leq$ stands for the inclusion partial order) 
are two states for $Y_{1:k-1}$, then 
$$
\P(Y_k = 1|Y_{1:k-1} = y_{1:k-1}) \geq \P(Y_k = 1|Y_{1:k-1} = z_{1:k-1}).
$$
In particular, $\forall k \in \{1,\dots, N \}$, if \, $\P(Y_{1:k-1} = 0_{1:k-1}) > 0$
then
\begin{align*}
q_k^*&= \P(Y_k = 1|Y_{1:k-1} = 0_{1:k-1})\\
&= K(k,k) + K_{k \times \{1:k-1\}} ((I-K)_{\{1:k-1\}})^{-1} K_{ \{1:k-1\}\times k}.
\end{align*}
\end{proposition}

\begin{proof}
Recall that by Proposition~\ref{prop_pointwise_conditional_proba_dpp},
$\displaystyle P(\{k\}\subset Y | A \subset Y,~B\cap Y = \emptyset)= H^B(k,k) - H^B_{\{k\}\times A} (H^B_A)^{-1} H^B_{A\times \{k\}}$.
Let $l \notin A\cup B\cup\{k\}$.
Consider $T^B$ the Cholesky decomposition of the matrix $H^B$ obtained with the following ordering the coefficients: $A$, $l$, the remaining coefficients of $\Y\setminus(A\cup\{l\})$.
Then, the restriction $T^B_A$ is the Cholesky decomposition (of the reordered) $H^B_A$ and thus
\begin{equation*}
\begin{aligned}
H^B_{\{k\}\times A} (H^B_A)^{-1} H^B_{A\times \{k\}}
&= H^B_{\{k\}\times A} (T^B_A(T^B_A)^*)^{-1} H^B_{A\times \{k\}}\\
&= \|(T^B_A)^{-1} H^B_{A\times \{k\}}\|_2^2.
\end{aligned}
\end{equation*}
Similarly,
$$
H^B_{\{k\}\times A\cup\{l\}} (H^B_{A\cup\{l\}})^{-1} H^B_{A\cup\{l\}\times \{k\}}
= \|(T^B_{A\cup\{l\}})^{-1} H^B_{A\cup\{l\}\times \{k\}}\|_2^2.
$$
Now note that solving the triangular system with $b = (T^B_{A\cup\{l\}})^{-1} H^B_{A\cup\{l\}\times \{k\}}$ amounts solving the triangular system 
with $(T^B_A)^{-1} H^B_{A\times \{k\}}$ and an additional line at the bottom.
Hence, one has $\|b\|_2^2\geq \|(T^B_A)^{-1} H^B_{A\times \{k\}}\|_2^2$.

Consequently, provided that $\P(A\cup\{l\} \subset Y,~B\cap Y = \emptyset)>0$,
$$
\P(\{k\}\subset Y | A\cup\{l\} \subset Y,~B\cap Y = \emptyset) 
\leq \P(\{k\}\subset Y | A \subset Y,~B\cap Y = \emptyset).
$$
The second inequality is obtained by complementarity in applying the above inequality to the DPP $\overline{Y}$ with $B\cup\{l\} \subset \overline{Y}$ and $A\cap \overline{Y} = \emptyset$.
\end{proof}

As a consequence, an admissible choice for the distribution of the Bernoulli process is

\begin{equation}
q_k =
\begin{cases}
\P(Y_k=1|Y_{1:k-1} = 0_{1:k-1}) & \text{if } \P(Y_{1:k-1} = 0_{1:k-1})>0,
\\ 
1 & \text{otherwise}.
\end{cases}
\label{eq:qk_proba_bernoulli_dominating_dpp}
\end{equation}

Note that if for some index $k$, $\P(Y_{1:k-1} = 0_{1:k-1})>0$ is not satisfied, 
then for all the subsequent indexes $l\geq k$, $q_l=1$, that is the Bernoulli process becomes degenerate and contains all the points after $k$.
In the remaining of this section, $X$ will denote a Bernoulli process with probabilities $(q_k)$ given by~\eqref{eq:qk_proba_bernoulli_dominating_dpp}.

As discussed in the previous section, in addition to being easily simulated, one would like the cardinality of $X$ to be close to the one of $Y$, the final sample.
The next proposition shows that this is verified if all the eigenvalues of $K$ are strictly less than $1$.

\begin{proposition}[$|X|$ is proportional to $|Y|$]
\label{prop_cardX_proportional_to_cardY}
Suppose that $P(Y = \emptyset) = \det(I-K) >0$ and denote by $\lambda_{\max}(K)\in[0,1)$ the maximal eigenvalue of $K$.
Then,
\begin{equation}
\mathbb{E}(|X|) \leq \left(1 + \frac{\lambda_{\max}(K)}{2 \left(1-\lambda_{\max}(K)\right)} \right) \mathbb{E}(|Y|).
\label{eq:cardX_proportional_to_cardY} 
\end{equation}	
\end{proposition}

\begin{proof}
We know that
$q_k = K(k,k) + K_{\{k\}\times\{1:k-1\}} ((I-K)_{\{1:k-1\}})^{-1}K_{\{1:k-1\}\times\{k\}}$, by Proposition~\ref{prop_pointwise_conditional_proba_dpp}.
Since 
$$
\|((I-K)_{\{1:k-1\}})^{-1}\|_{\mathcal{M}_{k-1}(\C)} 
= \tfrac{1}{1 - \lambda_{\max}(K_{\{1:k-1\}})}
$$ 
and $\lambda_{\max}(K_{\{1:k-1\}}) \leq \lambda_{\max}(K)$, one has
$$ K_{\{k\}\times\{1:k-1\}} ((I-K)_{\{1:k-1\}})^{-1}K_{\{1:k-1\}\times\{k\}}
 \leq  \tfrac{1}{1-\lambda_{\max}(K)}  \|K_{\{1:k-1\}\times\{k\}}\|_2^2.
$$
Summing all these inequalities gives
$$
\mathbb{E}(|X|)
\leq \Tr(K) + \tfrac{1}{1-\lambda_{\max}(K)} \sum_{k=1}^{N} \|K_{\{1:k-1\}\times\{k\}}\|_2^2.
$$
The last term is the Frobenius norm of the upper triangular part of $K$, hence
in can be bounded by $\frac{1}{2} \|K\|^2_{F} = \frac{1}{2} \sum_{j=1}^N \lambda_j(K)^2$.
Since $\lambda_j(K)^2 \leq \lambda_j(K)\lambda_{\max}(K)$, 
 $\sum_{j=1}^N \lambda_j(K)^2 \leq \lambda_{\max}(K) \Tr(K) = \lambda_{\max}(K) \mathbb{E}(|Y|)$.
\end{proof}

\begin{algorithm}[h!]
\begin{enumerate}
\item Compute sequentially the probabilities $\P(X_k=1)=q_k$ of the Bernoulli process $X$: 
\begin{itemize}[leftmargin=-0.1cm]
\item Compute the Cholesky decomposition $T$ of the matrix $I-K$.
\item For $k=1$ to $N$:
\begin{itemize}
\item If $q_{k-1} < 1$ (with the convention $q_0=0$), 
$$
q_k = K(k,k) + \|T_{\{1,\dots,k-1\}}^{-1} K_{\{1,\dots,k-1\}\times\{k\}}\|_2^2.
$$
\item Else, $q_{k}=1$.
\end{itemize}
\end{itemize}
\item Draw the Bernoulli process $X$. Let $m=|X|$ and $k_1 < k_2 < \dots < k_m$ be the points of $X$.
\item Apply the sequential thinning to the points of $X$: 
\begin{itemize}[leftmargin=-0.1cm]
\item Attempt to add sequentially each point of $X$ to $Y$:\\
Initialize $A\leftarrow \emptyset$ 
and $B\leftarrow \{1,\dots, k_1-1\}$. 
\\
For $j=1$ to $m$
\begin{itemize}
 \item If $j >1$, $B\leftarrow B \cup \{k_{j-1}+1,\dots, k_j-1\}.$
 \item Compute the conditional probability $p_{k_j} = \P(\{k_j\}\subset Y | A \subset Y,~B\cap Y = \emptyset)$ (see Formula~\eqref{eq_pointwise_conditional_proba_dpp}):
 \begin{itemize}
    \item Update $T^B$ the Cholesky decomposition of $(I-K)_B$ (see Appendix~\ref{sec:cholesky_update}).
    \item Compute $J^B = (T^B)^{-1} K_{B\times A\cup\{k_j\}}$.
    \item Compute $H^B_{A\cup\{k\}} = K_{A\cup\{k_j\}} + (J^B)^t J^B$.
    \item Compute $p_{k_j} = H^B(k_j,k_j) - H^B_{\{k_j\}\times A} (H^B_A)^{-1} H^B_{A\times \{k_j\}}$.
 \end{itemize} 
 \item Add $k_j$ to $A$ with probability $\frac{p_{k_j}}{q_{k_j}}$ or to $B$ otherwise.
\end{itemize}
\item Return $A$.
\end{itemize}
\end{enumerate}
\caption{Sequential thinning algorithm of a DPP with kernel $K$} 
\label{alg:sequential_thinning_dpp_bernoulli}
\end{algorithm}

We can now introduce the final sampling algorithm that we call sequential thinning algorithm (Algorithm \ref{alg:sequential_thinning_dpp_bernoulli}). It presents the different steps of our sequential thinning algorithm to sample a DPP of kernel $K$.
The first step is a preprocess that must be done only once for a given matrix $K$.
Step 2 is trivial and fast.
The critical point is to sequentially compute the conditional probabilities 
$p_k = \P(\{k\}\subset Y | A \subset Y,~B\cap Y = \emptyset)$ for each point of $X$.
Recall that in Algorithm~\ref{algo_sequential_sampling_dpp} 
we use a Cholesky decomposition of the matrix $(I-K)_B$ which is updated by adding a line each time a point is added in $B$.
Here, the inverse of the matrix $(I-K)_B$ is only needed when visiting a point $k\in X$, so one updates the Cholesky decomposition by a single block, where the new block corresponds to all indices added to $B$ in one iteration (see Appendix \ref{sec:cholesky_update}). The \textsc{Matlab} implementation used for the experiments is available online (\url{https://www.math-info.univ-paris5.fr/~claunay/exact_sampling.html}), together with a Python version of this code, using the PyTorch library. Note that, very recently, Guillaume Gautier \cite{Gautier_thesis_2020} proposed an alternative computation of the Bernoulli probabilities, that generate the dominating point process in the first step of Algorithm \ref{alg:sequential_thinning_dpp_bernoulli}, so that it only requires the diagonal coefficients of the Cholesky decomposition $T$ of $I-K$.

\subsection{Computational Complexity}
Recall that the size of the ground set $\Y$ is $N$ and the size of the final sample is $|Y|=n$. Both algorithms introduced in this paper have running complexities of order $O(N^3)$, as the spectral algorithm. Yet, if we get into the details, the most expensive task in the spectral algorithm is the computation of the eigenvalues and the eigenvectors of the kernel $K$. As this matrix is Hermitian, the common routine to do so is the reduction of $K$ to some tridiagonal matrix to which the QR decomposition is applied, meaning that it is decomposed into the product of an orthogonal matrix and an upper triangular matrix. When $N$ is large, the total number of operations is approximately $\frac{4}{3}N^3$ \cite{Trefethen_Numerical_Linear_Algebra_1997}. In Algorithms \ref{algo_sequential_sampling_dpp} and \ref{alg:sequential_thinning_dpp_bernoulli}, one of the most expensive operations is the Cholesky decomposition of several matrices. We recall that the Cholesky decomposition of a matrix of size $N\times N$ costs approximately $\frac{1}{3}N^3$ computations, when $N$ is large \cite{Mayers_Suli_Introduction_Numerical_Analysis_2003}. Concerning the sequential algorithm \ref{algo_sequential_sampling_dpp}, at each iteration $k$, the number of operations needed is of order $|B|^2|A| + |B||A|^2 + |A|^3$, where $|A|$ is the number of selected points at step $k$ so it's lower than $n$, and $|B|$ the number of unselected points, bounded by $k$. Then, when $N$ tends to infinity, the total number of operations in Algorithm \ref{algo_sequential_sampling_dpp} is lower than $\frac{n}{3}N^3 + \frac{n^2}{2}N^2 + n^3N$ or $O(nN^3)$, as in general $n \ll N$. Concerning Algorithm \ref{alg:sequential_thinning_dpp_bernoulli},  the sequential thinning from $X$, coming from Algorithm \ref{algo_sequential_sampling_dpp}, costs $O(n|X|^3)$. Recall that $|X|$ is proportional to $|Y|=n$ when the eigenvalues of $K$ are smaller than 1 (see Equation \eqref{eq:cardX_proportional_to_cardY}) so this step costs $O(n^4)$. Then, the Cholesky decomposition of $I-K$ is the most expensive operation in Algorithm \ref{alg:sequential_thinning_dpp_bernoulli} as it costs approximately $\frac{1}{3}N^3$. In this case, the overall running complexity of the sequential thinning algorithm is of order $\frac{1}{3}N^3$, which is 4 times less than the spectral algorithm. When some eigenvalues of $K$ are equal to 1, Equation \eqref{eq:cardX_proportional_to_cardY} does not hold anymore so, in that case, the running complexity of Algorithm \ref{alg:sequential_thinning_dpp_bernoulli} is only bounded by $O(nN^3)$.

We will retrieve this experimentally as, depending on the application or on the kernel $K$, this Algorithm \ref{alg:sequential_thinning_dpp_bernoulli} is able to speed up the sampling of DPPs. Note that in the previous computations, we have not taken into account the possible parallelization of the sequential thinning algorithm. As a matter of fact, the Cholesky decomposition is parallelizable \cite{George_parallelCholesky_1986}. Incorporating this parallel computations would probably speed up the sequential thinning algorithm, since the Cholesky decomposition of $I-K$ is the most expensive operation when the expected cardinality $|Y|$ is low. The last part of the algorithm, the thinning procedure, operates sequentially, so it is not parallelizable. These comments on the complexity and running times highly depends on the implementation, on the choice of the programming language and speed up strategies, so they mainly serve as an illustration.

\section{Experiments}

\subsection{DPP models for runtime tests}
\label{subsec_expe_kernels}
In the following section, we use the common notation of $L$-ensembles, with matrix $L = K(I-K)^{-1}$.
We present the results using four different kernels:
\begin{enumerate}[label=(\alph*)]
\item A random kernel: $ K = Q^{-1}DQ$, where 
$D$ is a diagonal matrix with uniformly distributed random values in $(0,1)$ and
$Q$ an unitary matrix created from the QR decomposition of a random matrix.
\label{item_random_kernel}
    
\item A discrete analog to the Ginibre kernel: $K = L(I+L)^{-1}$ with for all $x_1, x_2 \in \Y= \{1,\dots,N \}$,
$$L(x_1,x_2) = \frac{1}{\pi} e^{ -\frac{1}{2}( |x_1|^2 +|x_2|^2) + x_1 x_2},$$
\label{item_ginibre_kernel}

\item A patch-based kernel: 
Let $u$ be a discrete image and $\Y = \mathcal{P}$ a subset of all its patches, i.e. square sub-images of size $w\times w$ in $u$. Define $K = L(I+L)^{-1}$ where for all
$ P_1, P_2 \in \mathcal{P}$,
$$L(P_1,P_2) = \exp\left(-\frac{\|P_1-P_2\|_2^2}{s^2}\right)$$
where $s>0$ is called the bandwidth or scale parameter. 
We will detail the definition and the use of this kernel in Section~\ref{subsec_sampling_patches}.
\label{item_patch_kernel}

\item A projection kernel: $ K = Q^{-1}DQ$, where $D$ is a diagonal matrix with the $n$ first coefficients equal to 1, the others, equal to 0, and $Q$ is a random unitary matrix as for model \ref{item_random_kernel}.
\label{item_projection_kernel}

\end{enumerate}

It is often essential to control the expected cardinality of the point process. 
For case~\ref{item_projection_kernel} the cardinality is fixed to $n$.
For the three other cases, we use a procedure similar to the one developed in \cite{Barthelme_asymptDPPkDPP_2019}.
Recall that if $Y \sim \DPP(K)$ and $K = L(I+L)^{-1}$,  $\displaystyle \mathbb{E}(|Y|) = \tr(K) = \sum_{i \in \Y} \lambda_i = \sum_{i \in \Y} \frac{\mu_i}{1+\mu_i} $, where $(\lambda_i)_{i \in \Y}$ are the eigenvalues of $K$ and $(\mu_i)_{i \in \Y}$ are the eigenvalues of $L$ \cite{Hough_et_al_dpp_and_independence_2006,Kulesza_Taskar_dpp_for_machine_learning_2012}. 
Given an initial matrix $L = K(I-K)^{-1}$ and a desired expected cardinality $\mathbb{E}(|Y|) = n$, we run a binary search algorithm to find $\alpha > 0$ such that $\displaystyle \sum_{i \in \Y} \frac{\alpha \mu_i}{1+\alpha \mu_i} = n$. Then, we use the kernels $L_{\alpha} = \alpha L$ and $K_{\alpha} = L_{\alpha}(I+L_{\alpha} )^{-1}$.

\subsection{Runtimes}

For the following experiments, we ran the algorithms on a laptop HP Intel(R) Core(TM) i7-6600U CPU and we use the software \textsc{Matlab} R2018b. Note that the computational time results depend on the programming language and the use of optimized functions by the software. Thus, the following numerical results are mainly indicative.

First, let us compare the sequential thinning algorithm (Algorithm \ref{alg:sequential_thinning_dpp_bernoulli}) presented here with the two main sampling algorithms: the classic spectral algorithm (Algorithm \ref{alg_spectral_simulation_dpp}) and the ``naive'' sequential algorithm (Algorithm \ref{algo_sequential_sampling_dpp}). Figure \ref{fig:running_times_3algo} presents the running times of the three algorithms as a function of the total number of points of the ground set. Here, we have chosen a patch-based kernel~\ref{item_patch_kernel}. The expected cardinality $\mathbb{E}(|Y|)$ is constant, equal to $20$. As foreseen, the sequential algorithm (Algorithm \ref{algo_sequential_sampling_dpp}) is far slower than the two others. Whatever the chosen kernel and the expected cardinality of the DPP, this algorithm is not competitive.
Note that the sequential thinning algorithm uses this sequential method after sampling the particular Bernoulli process. But we will see that this first dominating step can be very efficient and lead to a relatively fast algorithm.\\

\begin{figure}[h!]
\centering
	\includegraphics[scale=0.4]{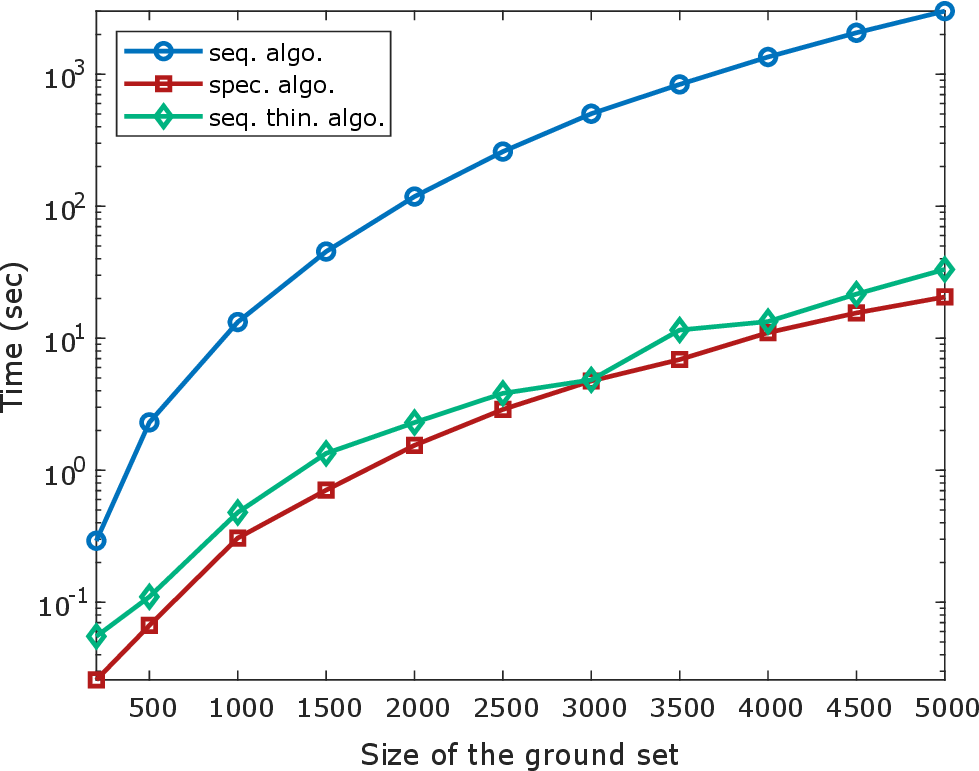}
	\caption{Running times of the 3 studied algorithms in function of the size of the ground set, using a patch-based kernel. 
	}
	\label{fig:running_times_3algo}
\end{figure}

\begin{figure*}[h]
\begin{center}

	\includegraphics[scale=0.195]{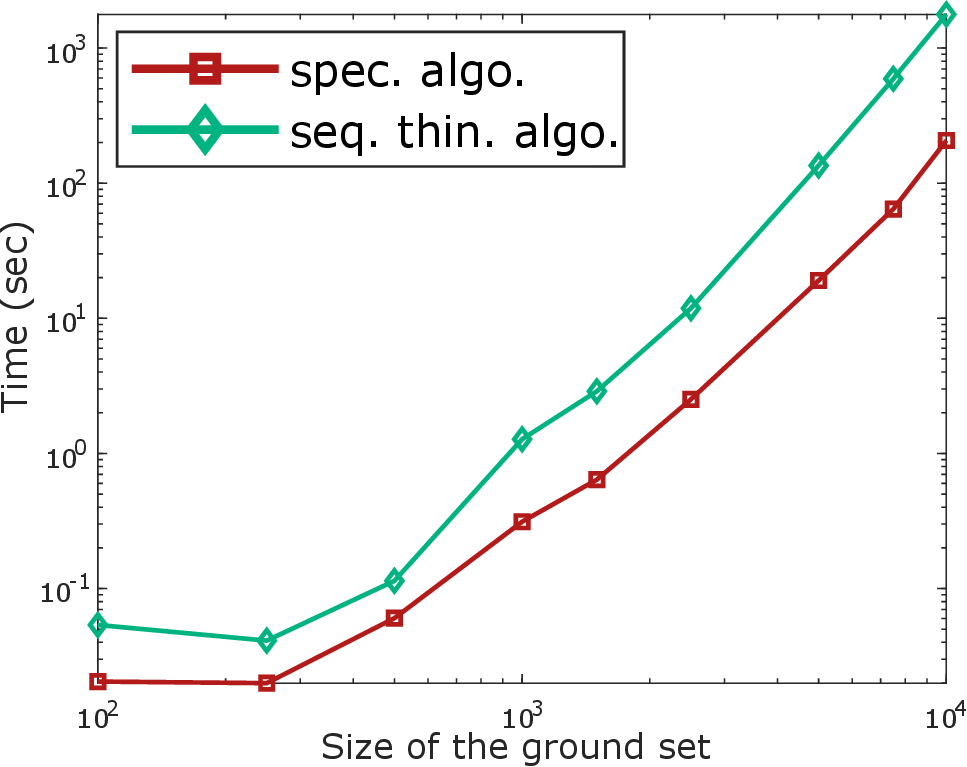}
	\hfill
	\includegraphics[scale=0.195]{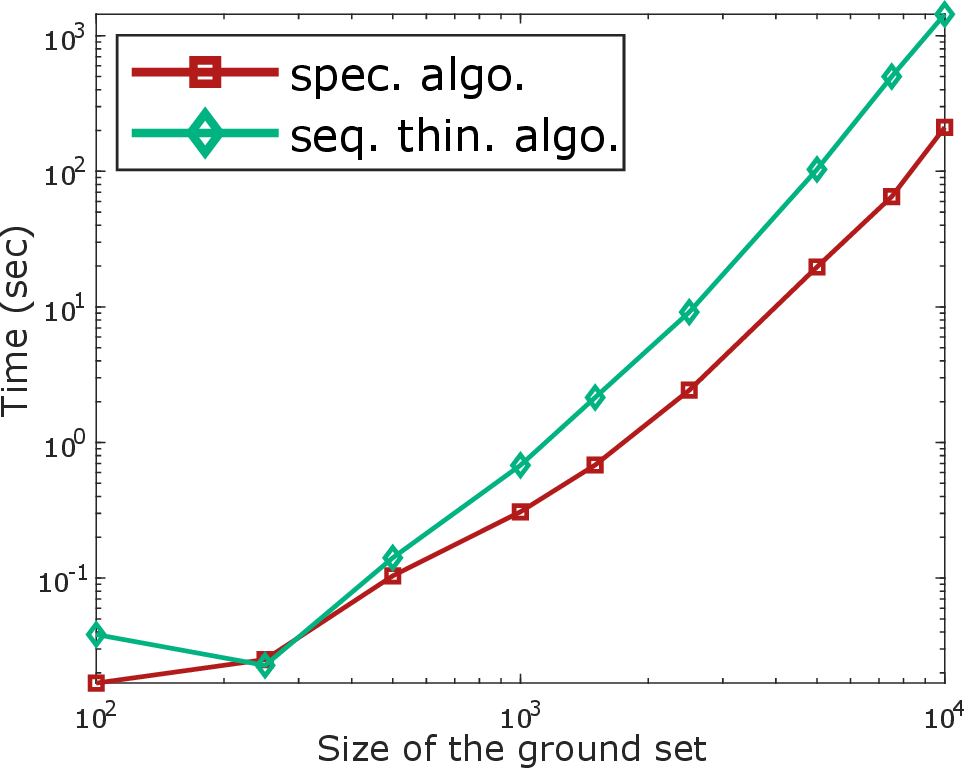}
	\hfill
	\includegraphics[scale=0.195]{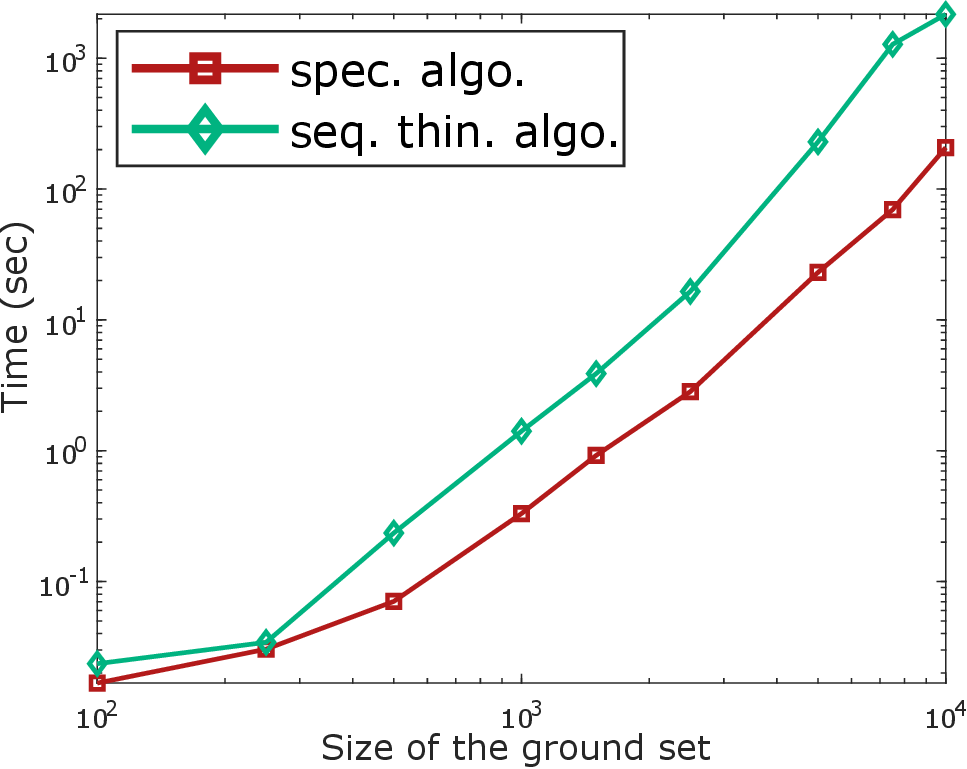}
	\hfill
	\includegraphics[scale=0.195]{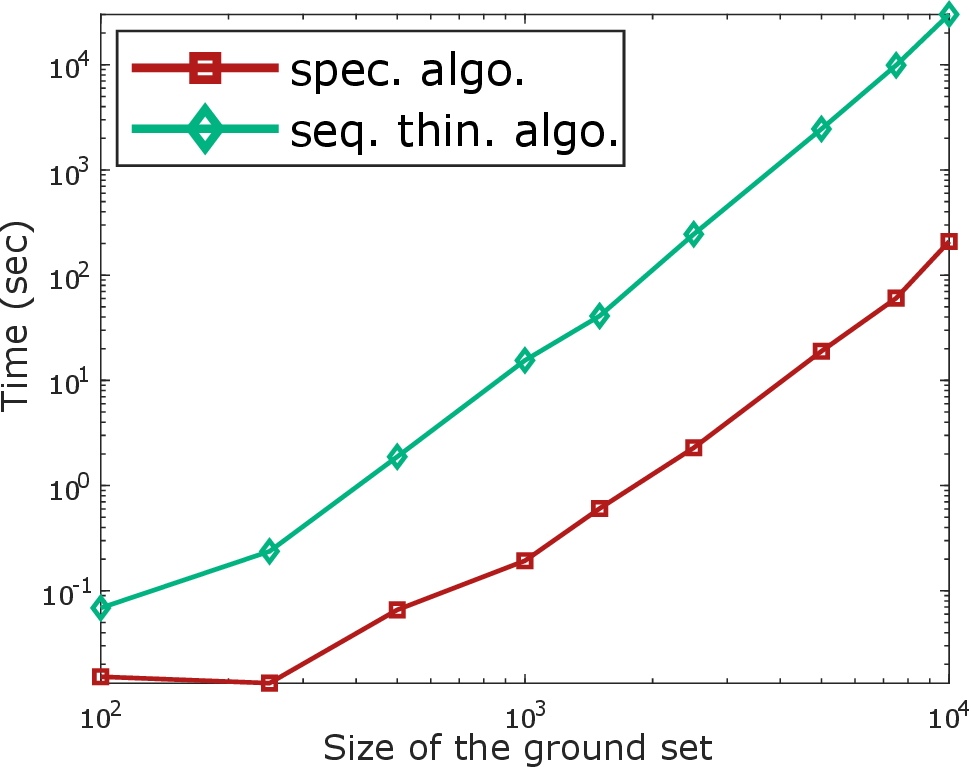}
\vspace{0.3cm}

	\includegraphics[scale=0.195]{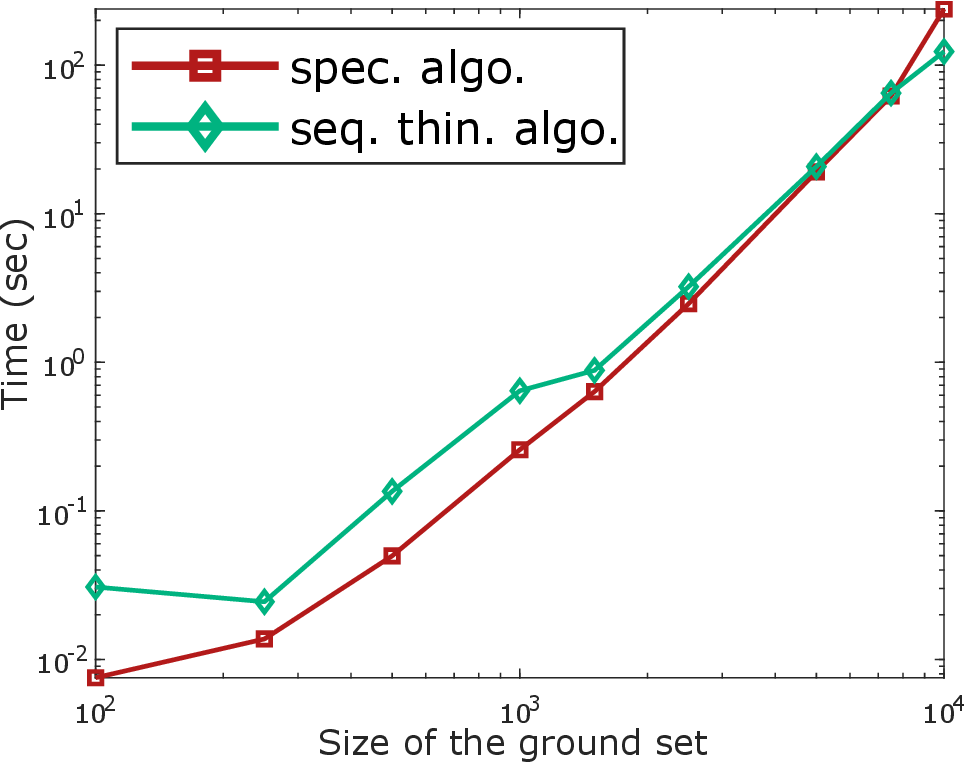} 
	\hfill
	\includegraphics[scale=0.195]{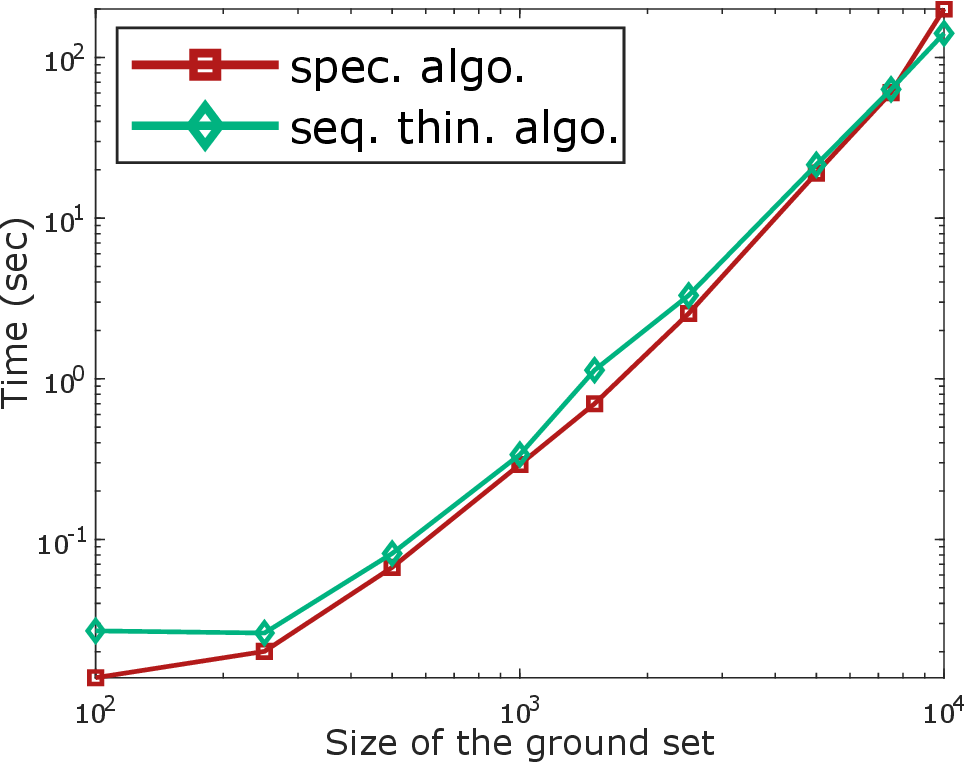}
	\hfill
	\includegraphics[scale=0.195]{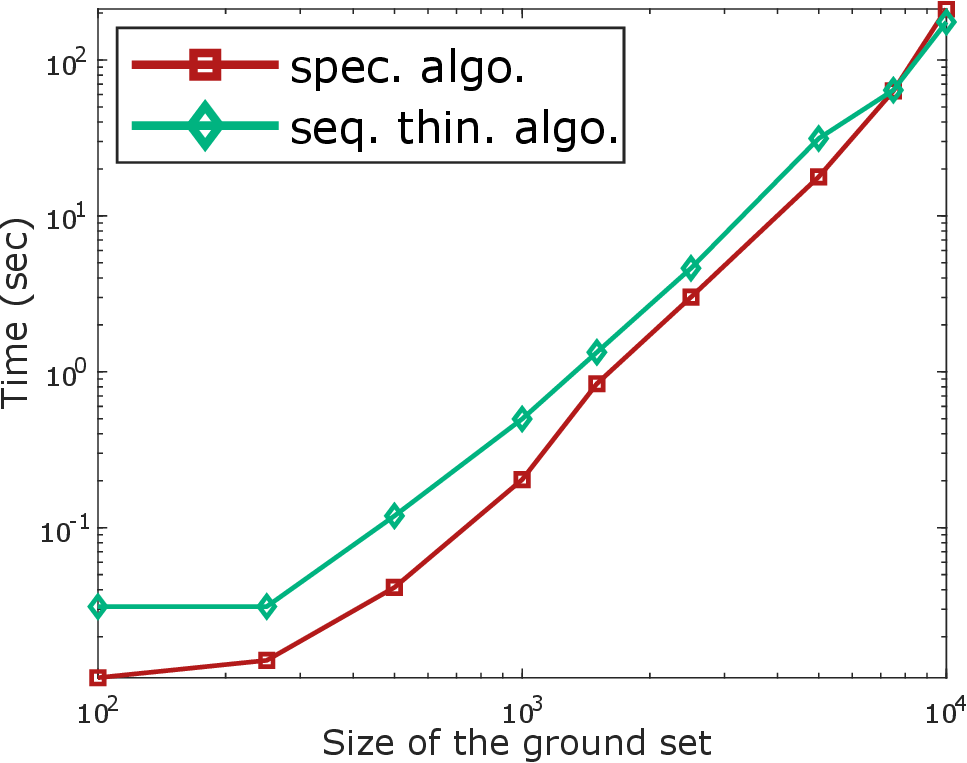}
	\hfill
	\includegraphics[scale=0.195]{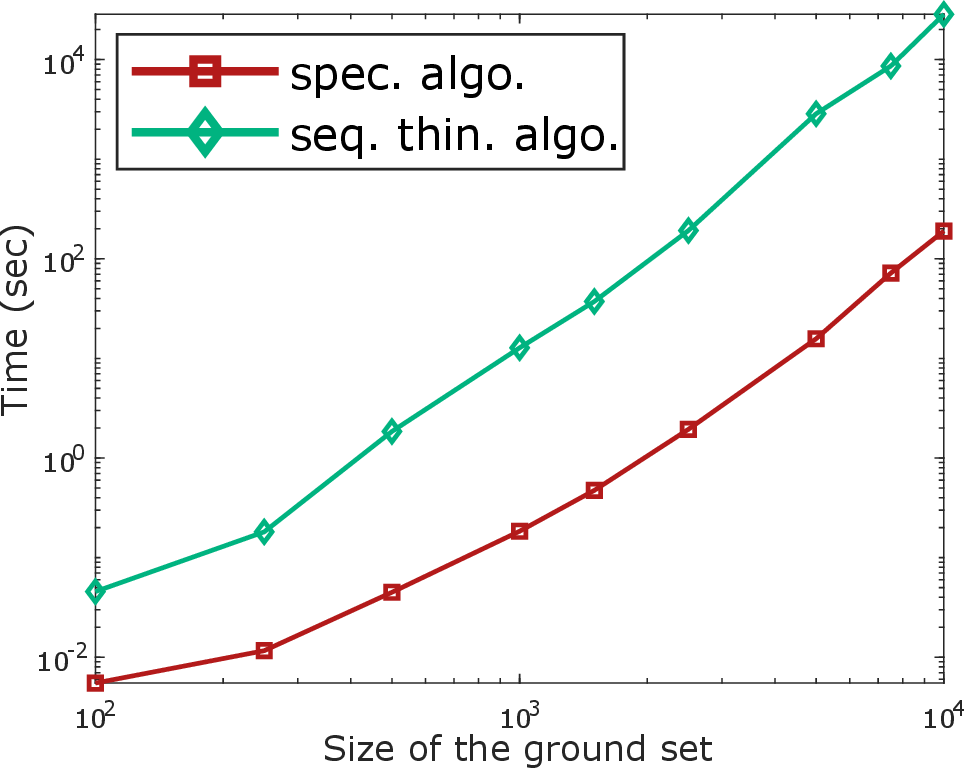}
\end{center}

\caption{Running times in log-scale of the spectral and the sequential thinning algorithms as a function of the size of the ground set $|\Y|$, using ``classic" DPP kernels. From left to right: a random kernel, a Ginibre-like kernel, a patch-based kernel and a projection kernel. On the first row, the expectation of the number of sampled points is set to $4 \%$ of $|\Y|$ and on the second row, $\mathbb{E}(|Y|)$ is constant, equal to $20$.}
\label{fig:running_times_2algo_growing_groundset}
\end{figure*}

From now on, we restrict the comparison to the spectral and the sequential thinning algorithms (Algorithms \ref{alg_spectral_simulation_dpp} and \ref{alg:sequential_thinning_dpp_bernoulli}). We present in Figure \ref{fig:running_times_2algo_growing_groundset} the running times of these algorithms as a function of the size of $|\Y|$ in various situations. The first row shows the running times when the expectation of the number of sampled point $\mathbb{E}(|Y|)$ is equal to $4\%$ of the size of $\Y$: it increases as the total number of points increases. In this case, we can see that whatever the chosen kernel, the spectral algorithm is faster as the complexity of sequential part of Algorithm \ref{alg:sequential_thinning_dpp_bernoulli} depends on the size $|X|$ that also grows. On the second row, as $|\Y|$ grows, $\mathbb{E}(|Y|)$ is fixed to $20$. Except for the right-hand-side kernel, we are in the configuration where $|X|$ stays proportional to $|Y|$, then the Bernoulli step of Algorithm \ref{alg:sequential_thinning_dpp_bernoulli} is very efficient and this sequential thinning algorithm becomes competitive with the spectral algorithm. 
For these general kernels, we observe that the sequential thinning algorithm can be as fast as the spectral algorithm, and even faster, when the expected cardinality of the sample is small compared to the size of the ground set. The question is: when and up to which expected cardinality is Algorithm \ref{alg:sequential_thinning_dpp_bernoulli} faster?

\begin{figure*}
\begin{center}

	\includegraphics[scale=0.195]{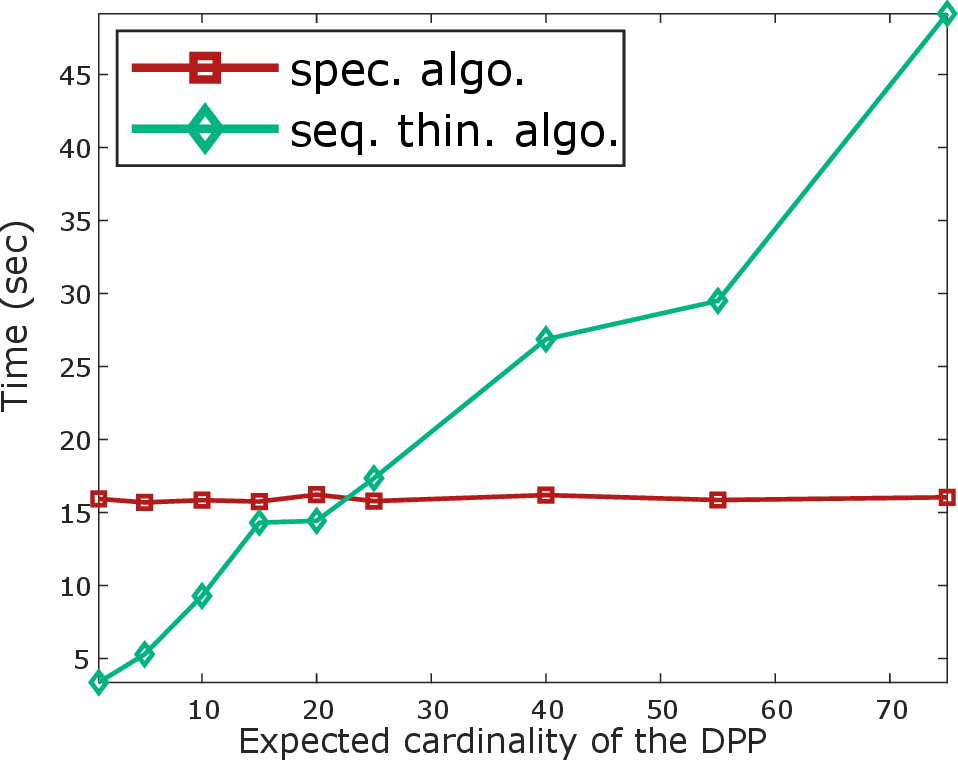}
	\hfill
	\includegraphics[scale=0.195]{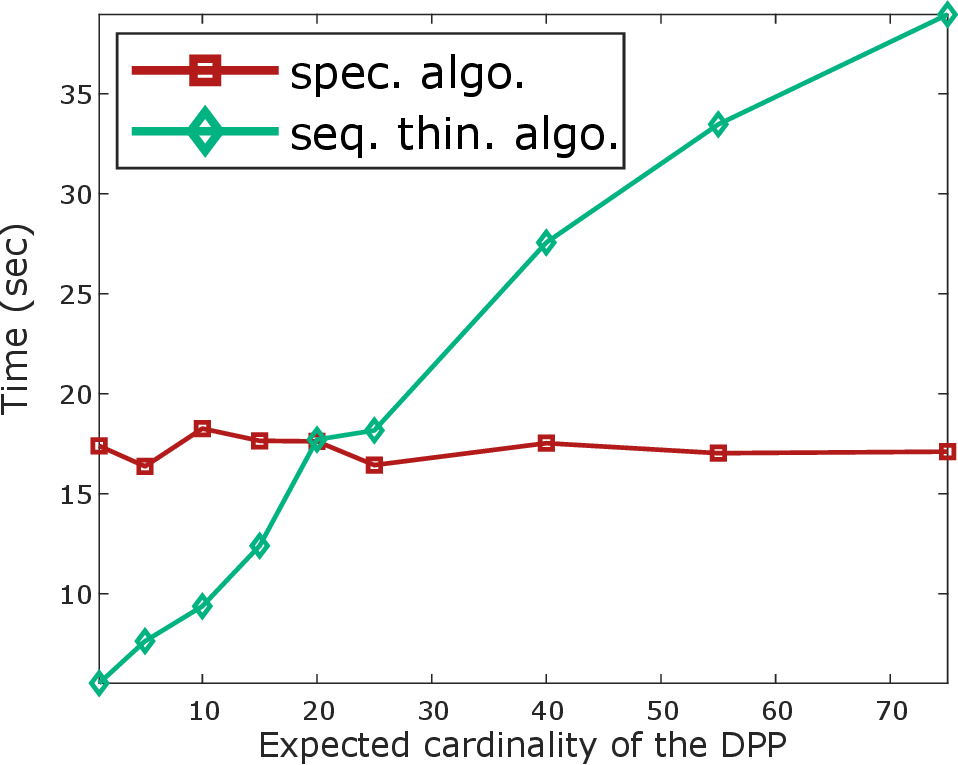}
	\hfill
	\includegraphics[scale=0.195]{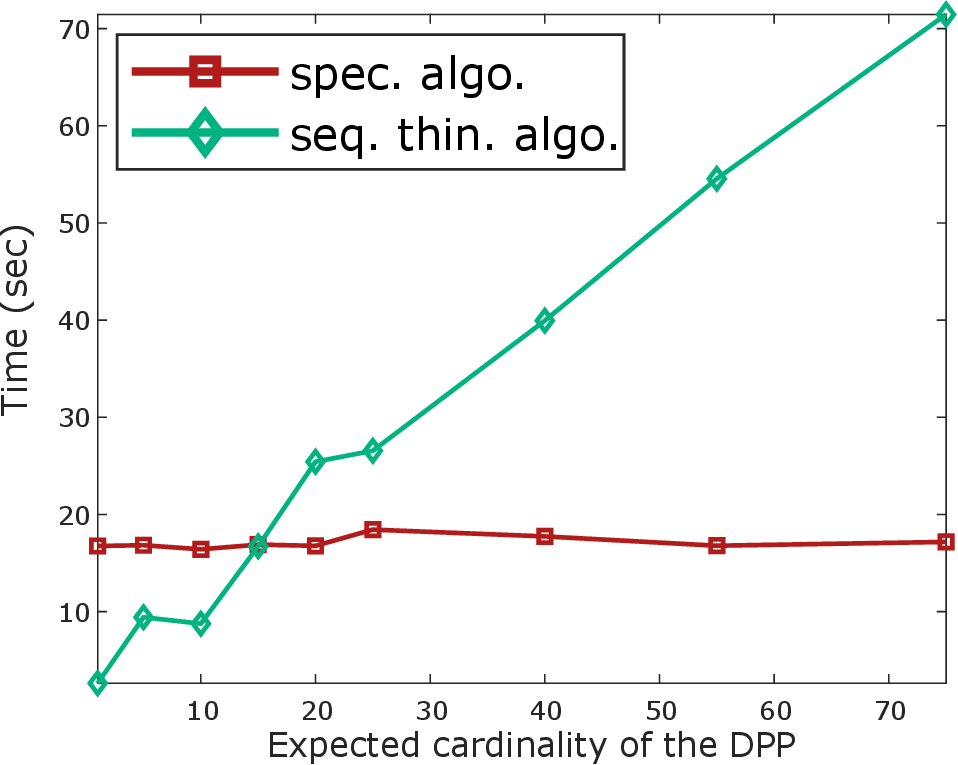}
	\hfill
	\includegraphics[scale=0.195]{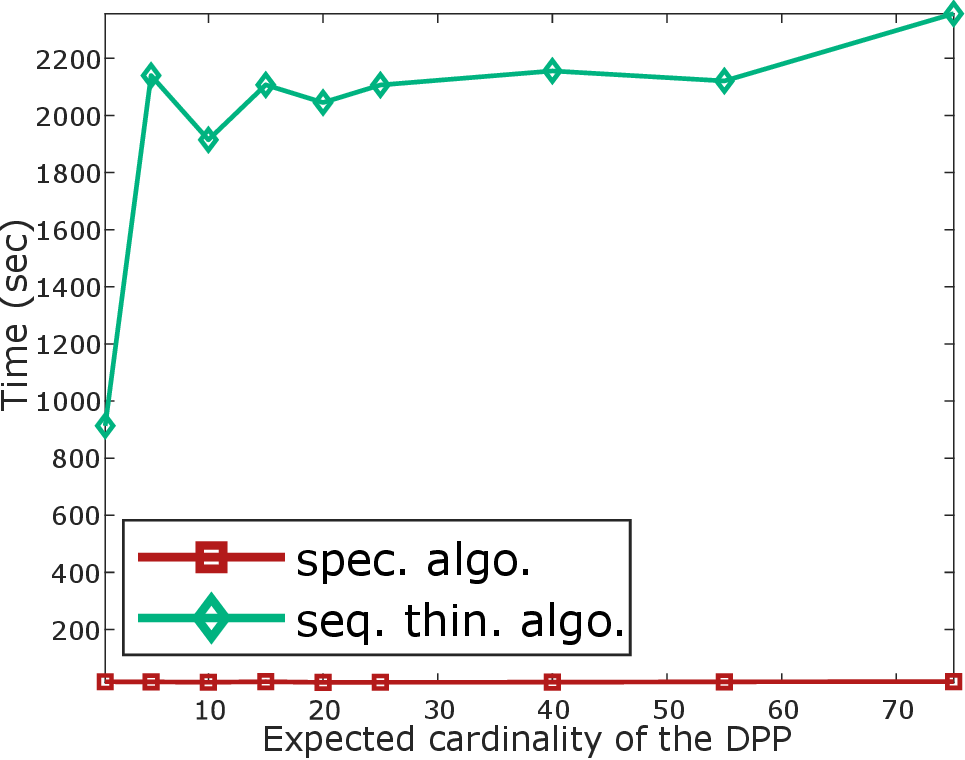}
\end{center}
\caption{Running times of the spectral and sequential thinning algorithms in function of the expected cardinality of the process. From left to right, using a random kernel, a Ginibre-like kernel, the patch-based kernel and a projection kernel. The size of the ground set is fixed to $5000$ in all examples.}
\label{fig:running_times_2algo_growing_cardinality}
\end{figure*}

Figure \ref{fig:running_times_2algo_growing_cardinality} displays the running times of both algorithms in function of the expected cardinality of the sample when the size of the ground set is constant, equal to 5000 points. Notice that, concerning the three left-hand-side general kernels with no eigenvalue equal to one, the sequential thinning algorithm is faster under a certain expected number of points -which depends on the kernel. For instance, when the kernel is randomly defined and the range of desired points to sample is below 25, it is relevant to use this algorithm. To conclude, when the eigenvalues of the kernel are below one, Algorithm \ref{alg:sequential_thinning_dpp_bernoulli} seems relevant for large data sets but small samples. This case is quite common, for instance to summarize a text, to work only with representative points in clusters or to denoise an image with a patch-based method.

The projection kernel (when the eigenvalues of $K$ are either $0$ or $1$) is, as expected, a complicated case. Figure \ref{fig:running_times_2algo_growing_groundset} (bottom, right) shows that our algorithm is not competitive when using this kernel. Indeed, the cardinality of the dominating Bernoulli process $X$ can be very large. In this case, the bound in Equation \eqref{eq:cardX_proportional_to_cardY} isn't valid (and even tends to infinity) as $\lambda_{\max} = 1$, and we necessarily reach the degenerated case when, after some index $k$, all the Bernoulli probabilities $q_l, l \geq k,$ are equal to 1. Then the second part of the sequential thinning algorithm -the sequential sampling part- is done on a larger set which significantly increases the running time of our algorithm. Figure \ref{fig:running_times_2algo_growing_cardinality} confirms this observation as in that configuration, the sequential thinning algorithm is never the fastest.

\begin{figure*}
\begin{minipage}{.31\textwidth}
	\centering
	\includegraphics[width=1\textwidth]{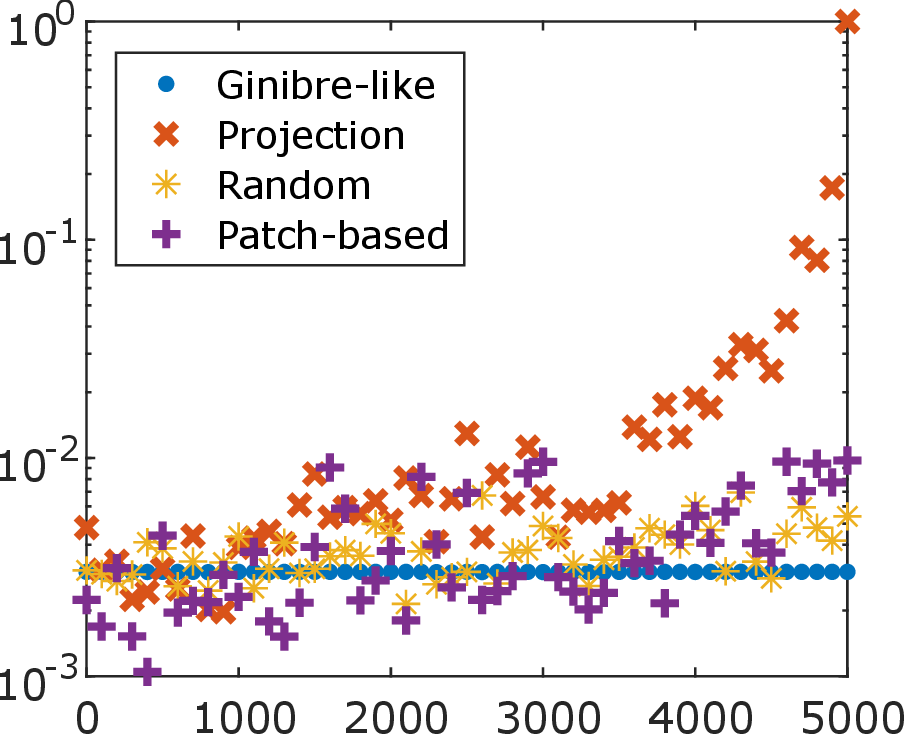}
	\subcaption{$\mathbb{E}(|Y|) = 15$}
\end{minipage} 
\hfill
\begin{minipage}{.31\textwidth}
	\centering
	\includegraphics[width=1\textwidth]{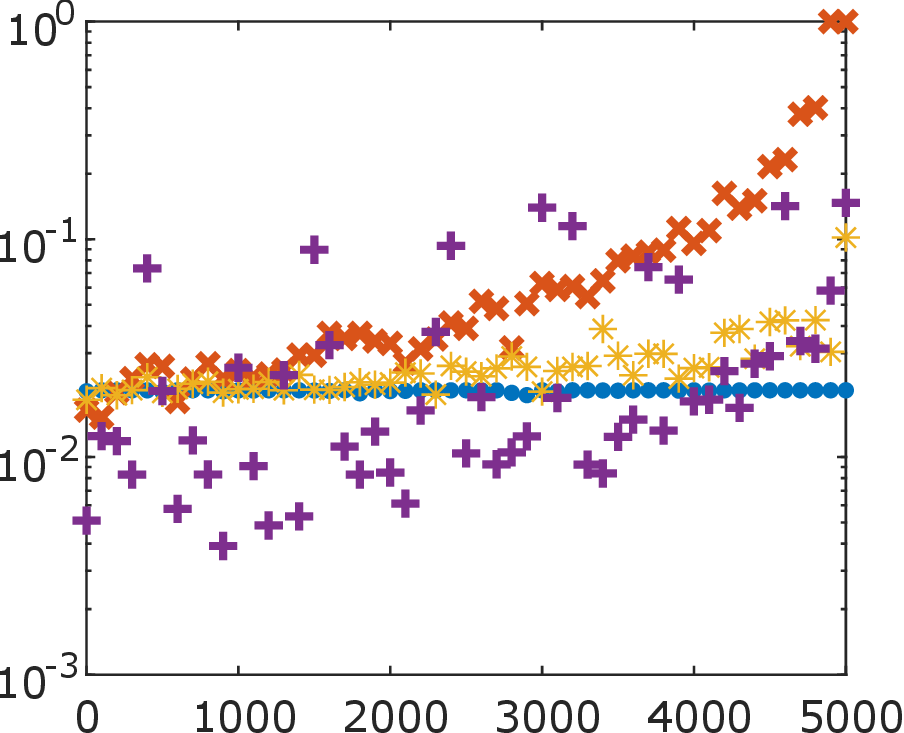}
	\subcaption{$\mathbb{E}(|Y|) = 100$}
\end{minipage} 
\hfill
\begin{minipage}{.31\textwidth}
	\centering
	\includegraphics[width=1\textwidth]{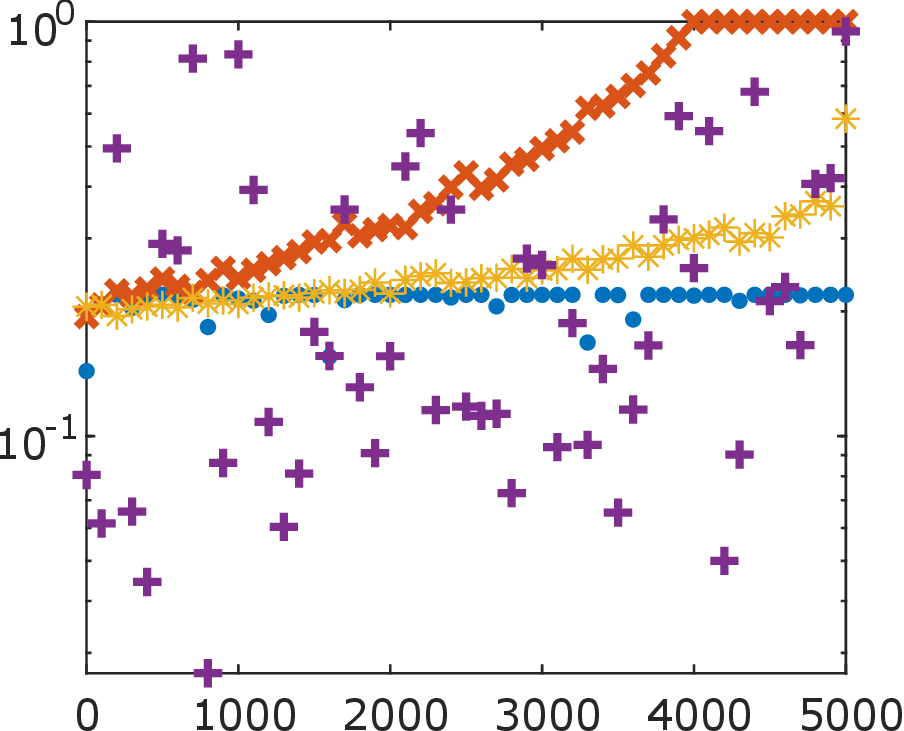}
	\subcaption{$\mathbb{E}(|Y|) = 1000$}
\end{minipage} 
\caption{Behavior of the Bernoulli probabilities $q_k, \, k \in \{1,\dots, N \}$, for the kernels presented in Section \ref{subsec_expe_kernels}, considering a ground set of $N= 5000$ elements and varying the expected cardinality of the DPP, $\mathbb{E}(|Y|) = 15,100,1000$.}
\label{fig:q_k}
\end{figure*}

Figure \ref{fig:q_k} illustrates how efficient the first step of Algorithm \ref{alg:sequential_thinning_dpp_bernoulli} can be to reduce the size of the initial set $\Y$. It displays Bernoulli probabilities $q_k, k \in \{1,\dots, N\}$ (Equation \ref{eq:qk_proba_bernoulli_dominating_dpp}) associated to the previous kernels, for different expected cardinality $\mathbb{E}(|Y|)$. Observe that the probabilities are overall higher for a projection kernel.  For such a kernel, we know that they necessarily reach the value 1, at the latest from the item $k = \mathbb{E}(|Y|)$. Indeed projection DPPs have a fixed cardinality (equal to $\mathbb{E}(|Y|)$) and $q_k$ computes the probability to select the item $k$ given that no other item has been selected yet. Notice that in general, considering the other kernels, the degenerated value $q_k=1$ is rarely reached, even though in our experiments, the Bernoulli probabilities associated to the patch kernel (c) are sometimes close to one, when the expected size of the sample is $\mathbb{E}(|Y|) = 1000$. On the opposite, the Bernoulli probabilities associated to the Ginibre-like kernel remain rather close to a uniform distribution.

In order to understand more precisely to what extent high eigenvalues penalize the efficiency of the sequential thinning algorithm (Algorithm \ref{alg:sequential_thinning_dpp_bernoulli}), Figure \ref{fig:eigenvalues} compares its running times with that of the spectral algorithm (Algorithm \ref{alg_spectral_simulation_dpp}) in function of the eigenvalues of the kernel $K$. For these experiments, we consider a ground set of size $|\Y| = 5000$ items and an expected cardinality equal to $15$. In the first case (a), the eigenvalues are either equal to 0 or to $\lambda_{\text{max}}$, whith $m$ non-zero eigenvalues so that  $m \lambda_{\text{max}} = 15$. It shows that above a certain $\lambda_{\text{max}}$ $(\simeq 0.65)$, the sequential thinning algorithm is not the fastest anymore. In particular, when $\lambda_{\text{max}} = 1$, the running time takes off. In the second case (b), the eigenvalues $(\lambda_k)$ are randomly distributed between 0 and $\lambda_{\text{max}}$ so that $\sum_k \lambda_k = 15$. In practice, $(N-1)$ eigenvalues are exponentially distributed, with expectation $\frac{15- \lambda_{\text{max}}}{N-1}$, and the last eigenvalue is set to $\lambda_{\text{max}}$. In this case, the sequential thinning algorithm remains faster than the spectral algorithm, even with high values of $\lambda_{\text{max}}$, except when $\lambda_{\text{max}} = 1$. This can be explained by the fact that, by construction of this kernel, most of the eigenvalues are very small. The average size of the Bernoulli process generated (light grey, right axes) also illustrates the influence of the eigenvalues.

\begin{figure*}
\begin{minipage}{.45\textwidth}
	\centering
	\includegraphics[width=1\textwidth]{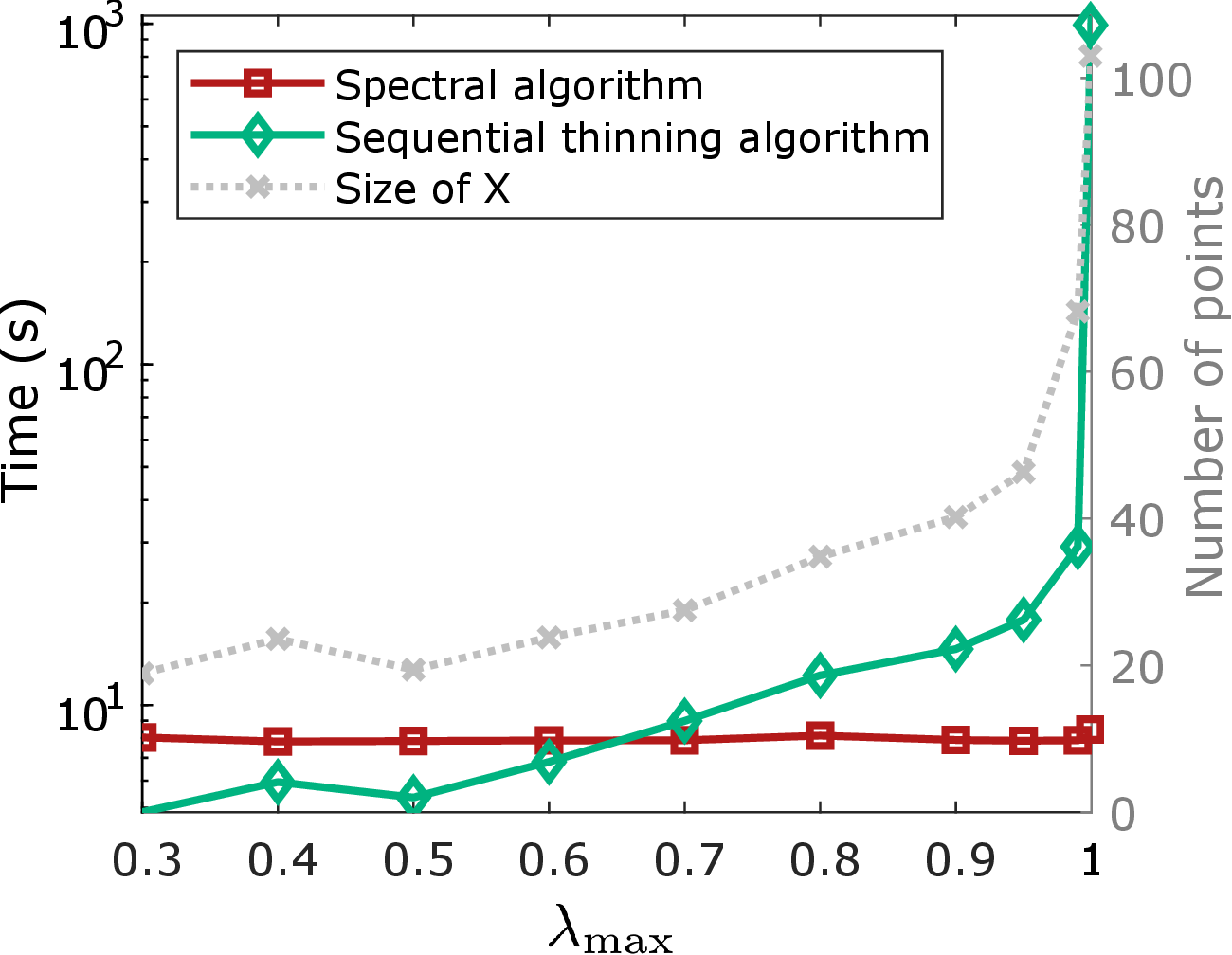}
	\subcaption{$m$ eigenvalues equal to $\lambda_{\text{max}}$ and $N-m$ zero eigenvalues.}
\end{minipage} 
\hfill
\begin{minipage}{.45\textwidth}
	\centering
	\includegraphics[width=1\textwidth]{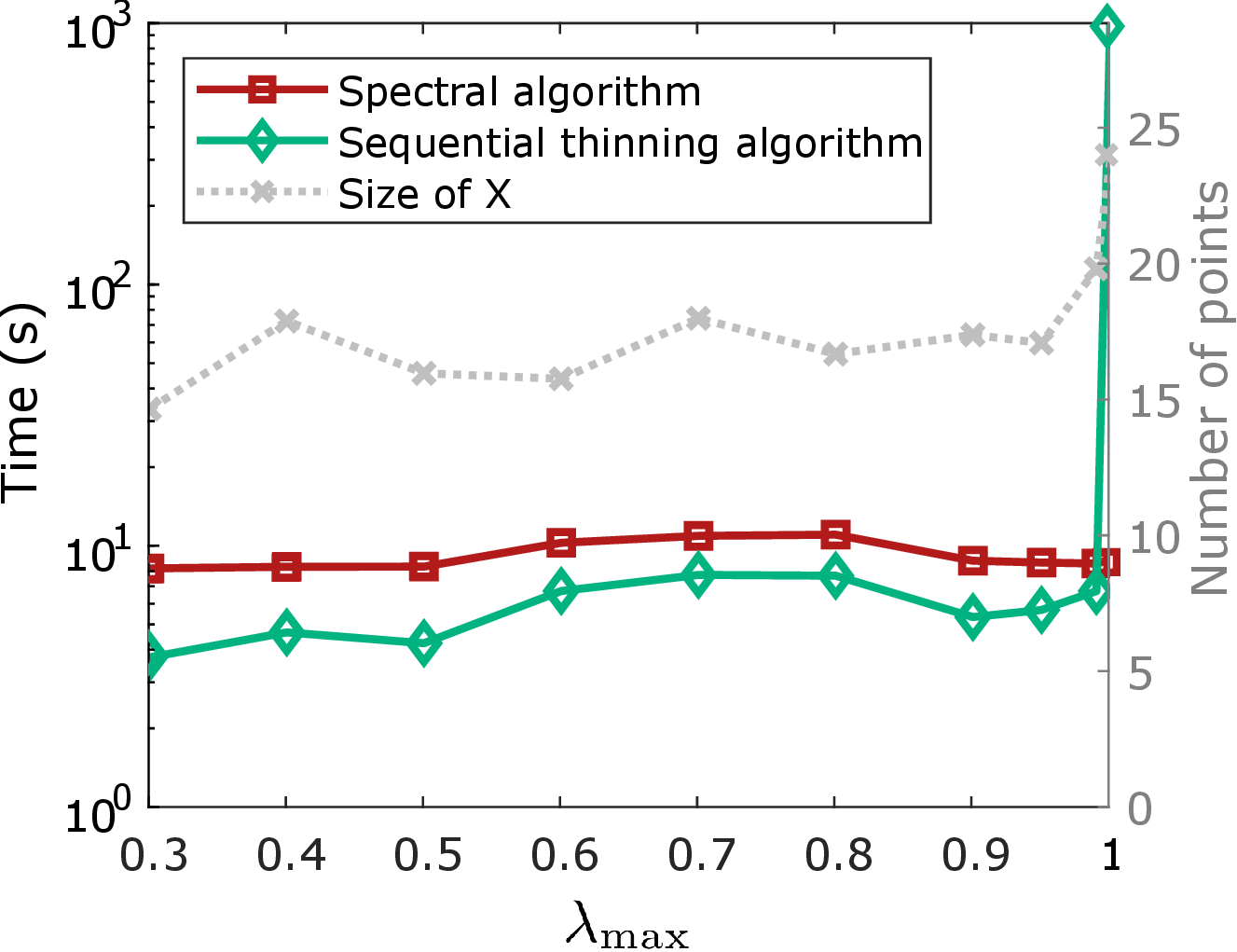}
	\subcaption{$N$ random eigenvalues between 0 and $\lambda_{\text{max}}$.}
\end{minipage} 
\caption{Running times of the spectral and sequential thinning algorithms (Algorithm \ref{alg_spectral_simulation_dpp} and \ref{alg:sequential_thinning_dpp_bernoulli}) in function of $\lambda_{\text{max}}$. The size of the Bernoulli process $X$ is also displayed in light grey (right axis). Here, $|\Y| = 5000$ and $\mathbb{E}(|Y|) = 15$.}
\label{fig:eigenvalues}
\end{figure*}

\begin{table}[h]
\centering

\begin{tabular}{llll}
\toprule 
Algorithms & Steps &  \multicolumn{2}{c}{Expected cardinality} \\ 
 
 & & $4 \%$ of $|\Y|$ & Constant (20) \\
 \midrule
Sequential & Matrix inversion & $ 74.25\%$ & $72.71 \%$ \\ 
 
  & Cholesky computation & $ 22.96\%$ & $17.82 \%$ \\ 
\midrule
Spectral & Eigendecomposition & $ 83.34\%$ & $94.24 \%$\\ 

  & Sequential sampling & $ 14.77\%$& $4.95 \%$ \\ 
\midrule 
Sequential thinning & Preprocess to define $q$ & $ 10.07\%$           & $13.43 \%$ \\ 
 
  & Sequential sampling & $ 89.39\%$  & $86.53 \%$ \\ 
\bottomrule
\end{tabular} 
\caption{Detailed running times of the sequential, spectral and sequential thinning algorithms for varying ground sets $\Y$ with $|\Y| \in [100,5000]$ using a patch-based kernel.}
\label{tab:detailed_running_times}
\end{table}

Table \ref{tab:detailed_running_times} presents the individual weight of the main steps of the three algorithms. Concerning the sequential algorithm, logically, the matrix inversion is the heaviest part taking $74.25\%$ of the global running time. These proportions remain the same when the expected number of points $n$ grows. The main operation of the spectral algorithm is by far the eigendecomposition of the matrix $K$, counting for $83\%$ of the global running time, when the expectation of the number of points to sample evolves with the size of $\Y$. Finally, the sequential sampling is the heaviest step of the sequential thinning algorithm. We have already mentioned that the thinning is very fast and that it produces a point process with a cardinality as close as possible to the final DPP. When the expected cardinality is low, the number of selected points by the thinning process is low too, so the sequential sampling part remains bounded ($86.53 \%$ when the expected cardinality $\mathbb{E}(|Y|)$ is constant). On the contrary, when $\mathbb{E}(|Y|)$ grows, the number of points selected by the dominated process rises as well so the running time of this step is growing (with a mean of $89.39\%$). As seen before, the global running time of the sequential thinning algorithm really depends on how good the domination is.

Thus, the main case when this sequential thinning algorithm (Algorithm \ref{alg:sequential_thinning_dpp_bernoulli}) fails to compete with the spectral algorithm (Algorithm \ref{alg_spectral_simulation_dpp}) is when the eigenvalues of the kernel are equal or very close to 1. This algorithm improves the sampling running times when the target size of the sample is very low (below 25 in our experiments).

In cases when multiple samples of the same DPP have to be drawn, 
the eigendecomposition of $K$ can be stored and the spectral algorithm is more efficient than ours.
Indeed, in our case the computation of the Bernoulli probabilities can also be saved but the sequential sampling is the heaviest task and needs to be done for each sample.

\subsection{Sampling the patches of an image}
\label{subsec_sampling_patches}

A random and diverse subselection of the set of patches of an image can be useful for numerous image processing applications. A first obvious one is image compression. Indeed, it is possible to obtain a good reconstruction of the image from a very small portion of its patches. It is sometimes necessary to keep only the most informative patches of the image, if possible a small amount, and reconstruct the image, store it, only using these few patches. Moreover, most of patch-based algorithms could use such a subselection of patches to improve or at least speed up its procedures, e.g. for denoising \cite{Buades_NLmeans_2005}.
To do this, the selected patches must be representative of the patches diversity and this is what DPPs offer. Launay and Leclaire \cite{LaunayLeclaire_textoDPP_2019} explore this strategy to speed up a texture synthesis algorithm.

Given an image $u$ and a set $\mathcal{P}$ of 10 000 randomly picked patches of $u$, we compare here the selection strategies using either a DPP or a random uniform selection.
Let us recall the patch-based kernel~\ref{item_patch_kernel} defined as the $L$-ensemble associated with
$$\forall P_1, P_2 \in \mathcal{P}, \quad L(P_1,P_2) = \exp\left( -\frac{\|P_1-P_2\|_2^2}{s^2}\right),$$
that is, $L$ is a Gaussian kernel applied to the Euclidean distance between the patches of $\mathcal{P}$.
This function is commonly chosen to define a similarity measure between patches. 
It is relevant since in general the reconstruction error is computed in function of the Euclidean distance between the original image and the reconstructed image. 
We set the bandwidth or scale parameter $s$ to be proportional to the median of the interdistances between the patches, as advised by Aggarwal~\cite{aggarwal2016gaussiankernel} and Tremblay et al.~\cite{Tremblayetal_DPPforCoresets_2018}.

Figure~\ref{fig:reconstruction} presents several reconstructions of two images, obtained by uniform selection or by the DPP defined above, with various expected sample sizes. Notice that while we can control the exact cardinality of the uniform selections, the number of patches in the DPP selections varies as we can only control the expected cardinality during the sampling process. This figure shows how a selection from a DPP provides better reconstructions than a uniform selection, especially when the number of patches is low. Indeed, as the DPP sampling favors diverse set of patches, it is less likely to miss an essential information of the image. 
On the contrary, nothing prevents the uniform selection from selecting very similar patches.
The Pool image on the bottom of Figure \ref{fig:reconstruction}, for a cradinality equal to 5, clearly illustrates this. The number of patches in an image depends on the size of the image and is often higher than 10000 while the selection needs to be small (between 5 and 100): here the use of our sequential thinning algorithm is pertinent.

\newlength{\mylength}
\setlength{\mylength}{0.147\columnwidth}
\begin{figure}
  \begin{center}
    \begin{tabular}{cccc}
      \includegraphics[width=\mylength]{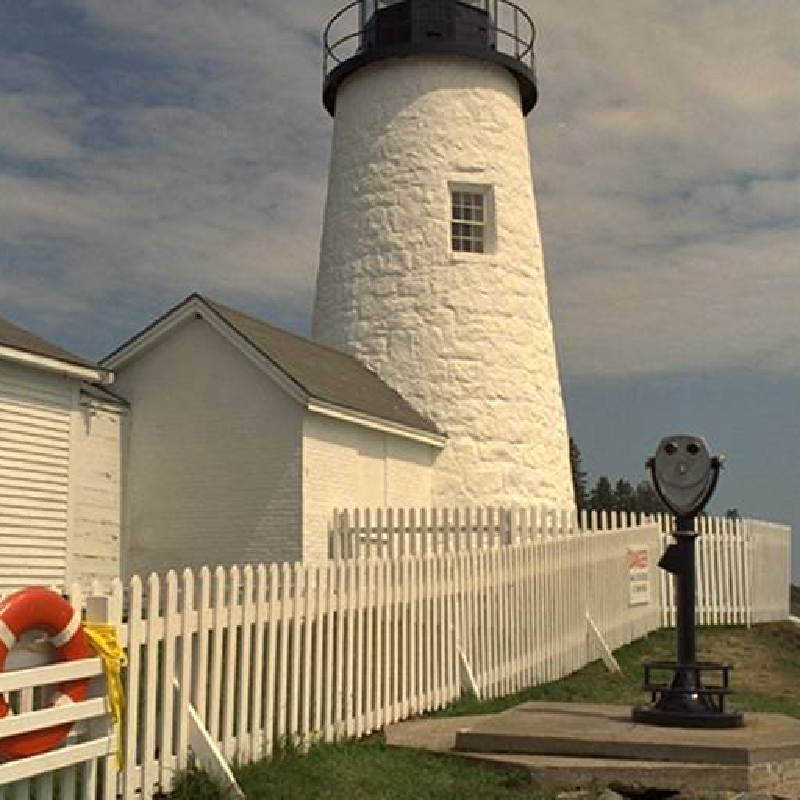} &
      \includegraphics[width=\mylength]{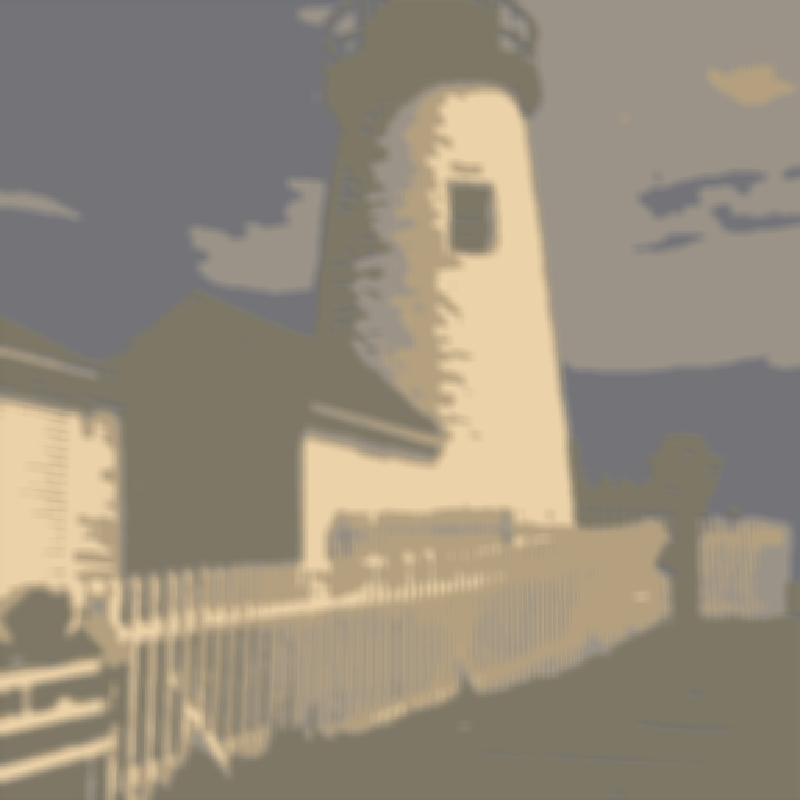} &
      \includegraphics[width=\mylength]{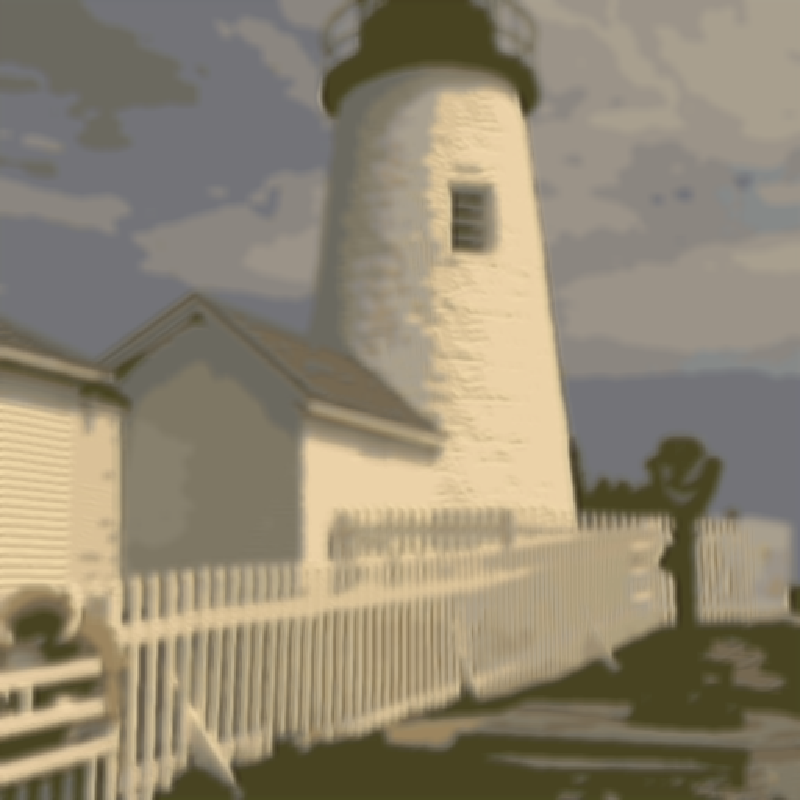} &
      \includegraphics[width=\mylength]{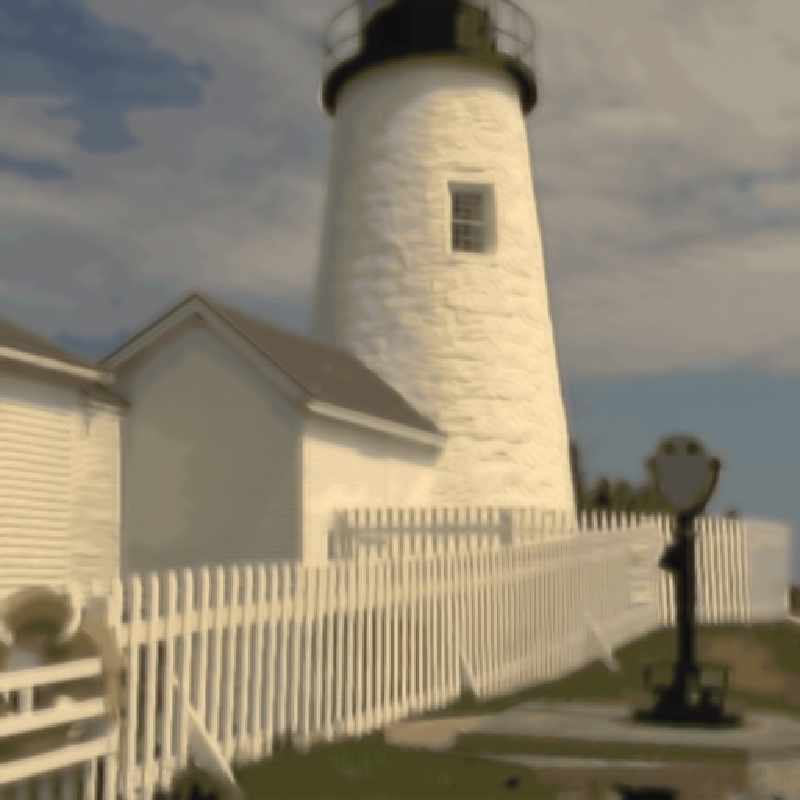} \\
      
       &
      \includegraphics[width=\mylength]{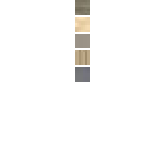} &
      \includegraphics[width=\mylength]{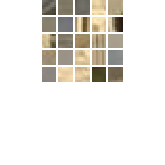} &
      \includegraphics[width=\mylength]{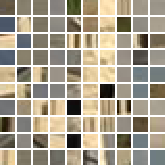}\\
      
       &
      \includegraphics[width=\mylength]{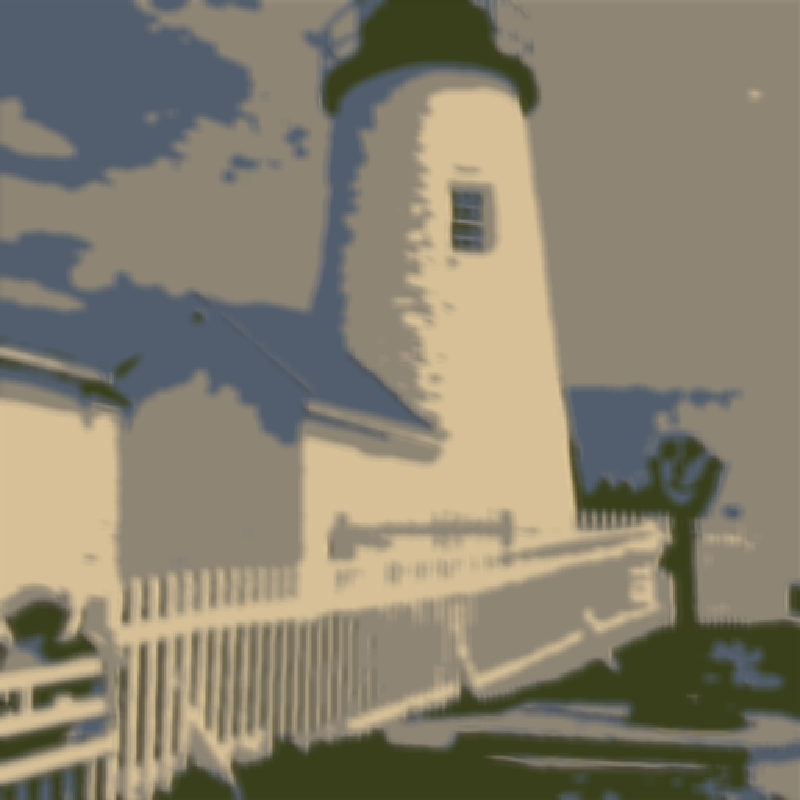} &
      \includegraphics[width=\mylength]{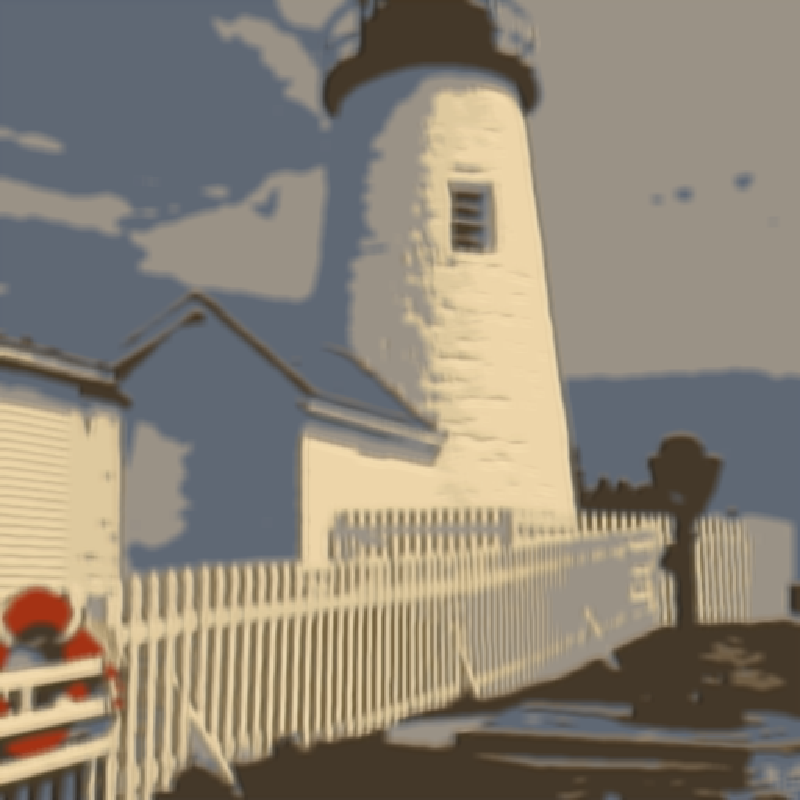} &
      \includegraphics[width=\mylength]{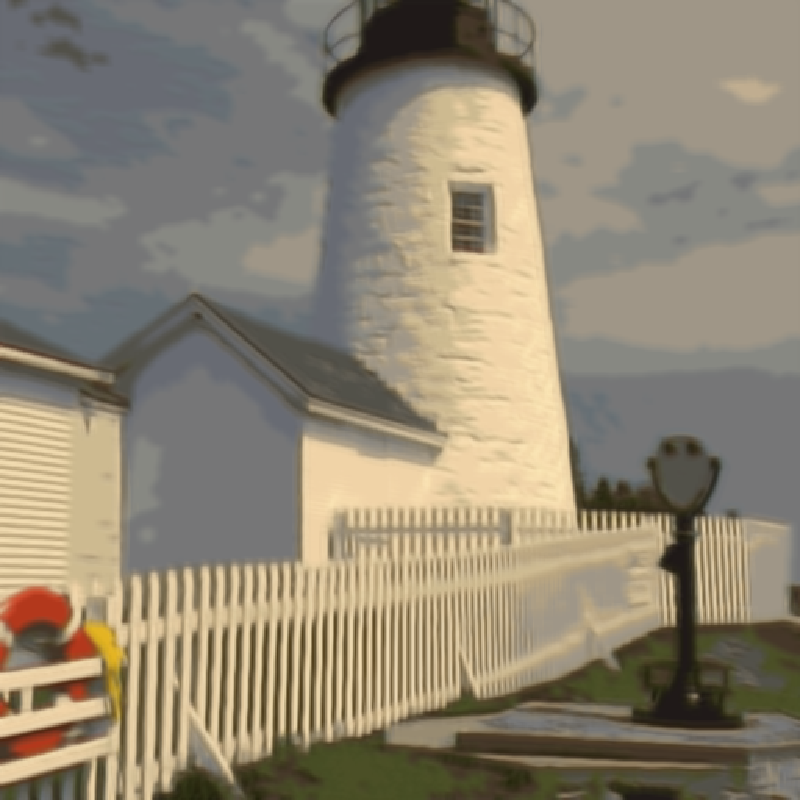} \\
      
       &
      \includegraphics[width=\mylength]{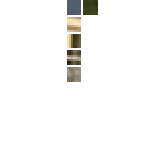} &
      \includegraphics[width=\mylength]{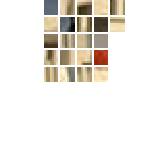} &
      \includegraphics[width=\mylength]{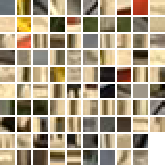} \\

      \includegraphics[width=\mylength]{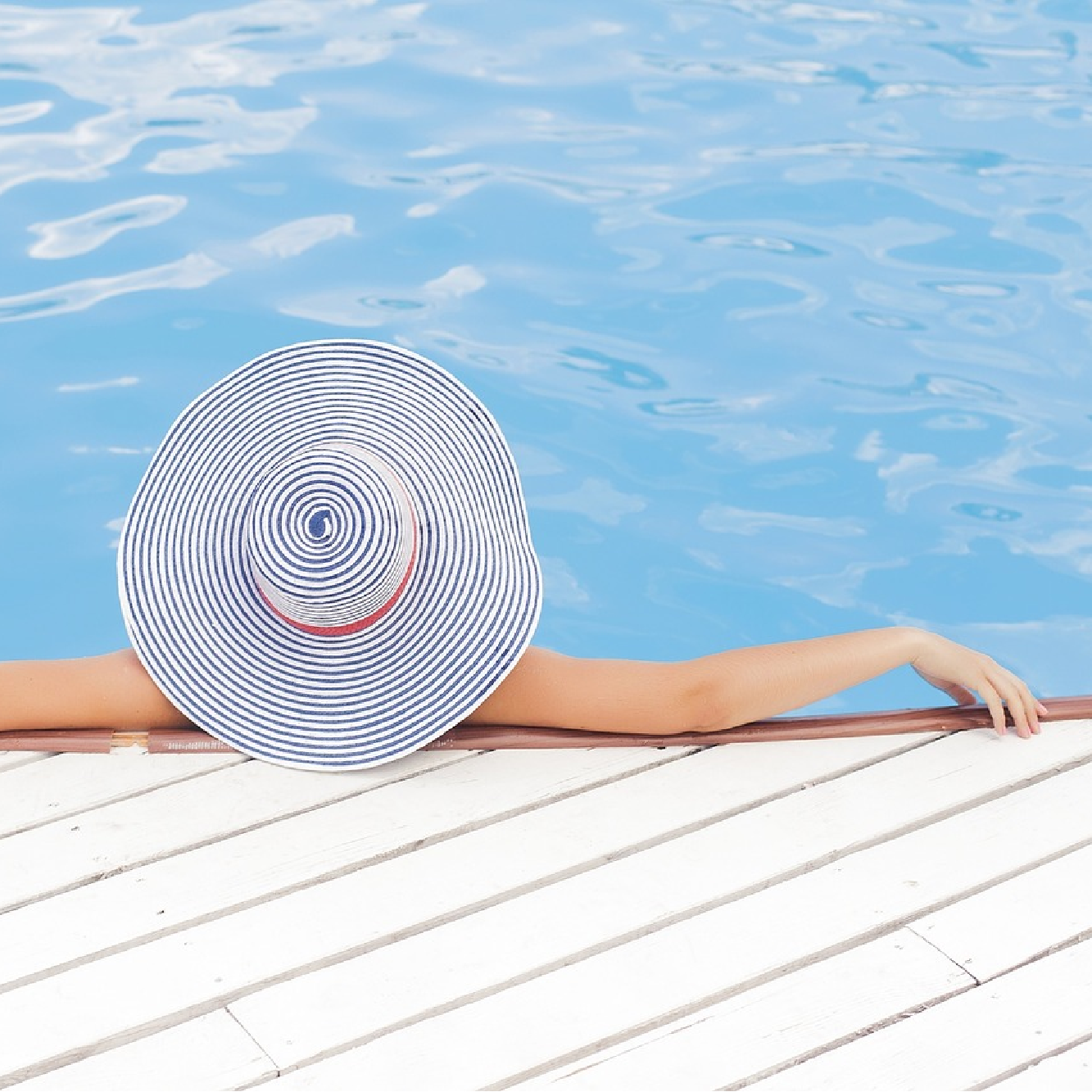} &
      \includegraphics[width=\mylength]{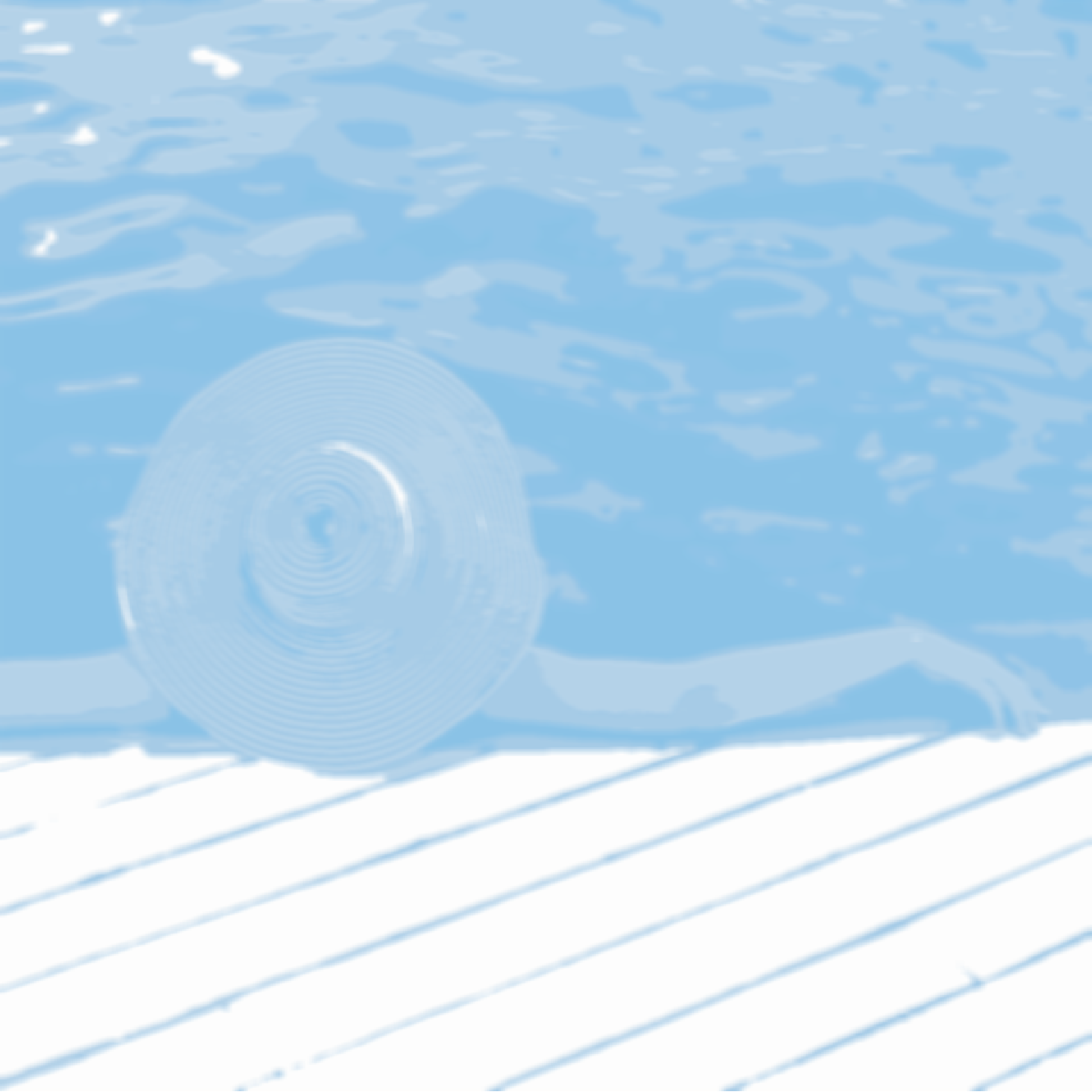} &
      \includegraphics[width=\mylength]{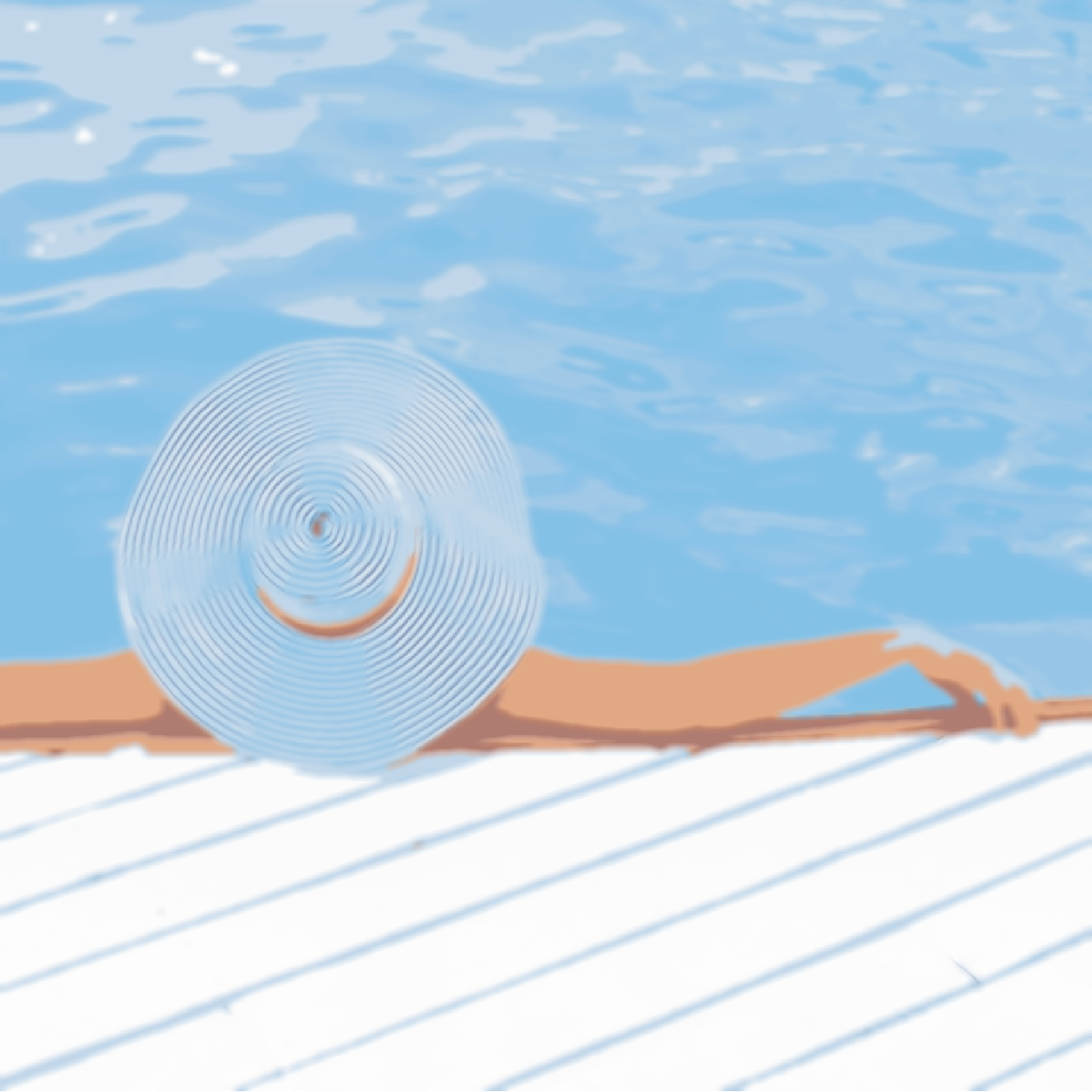} &
      \includegraphics[width=\mylength]{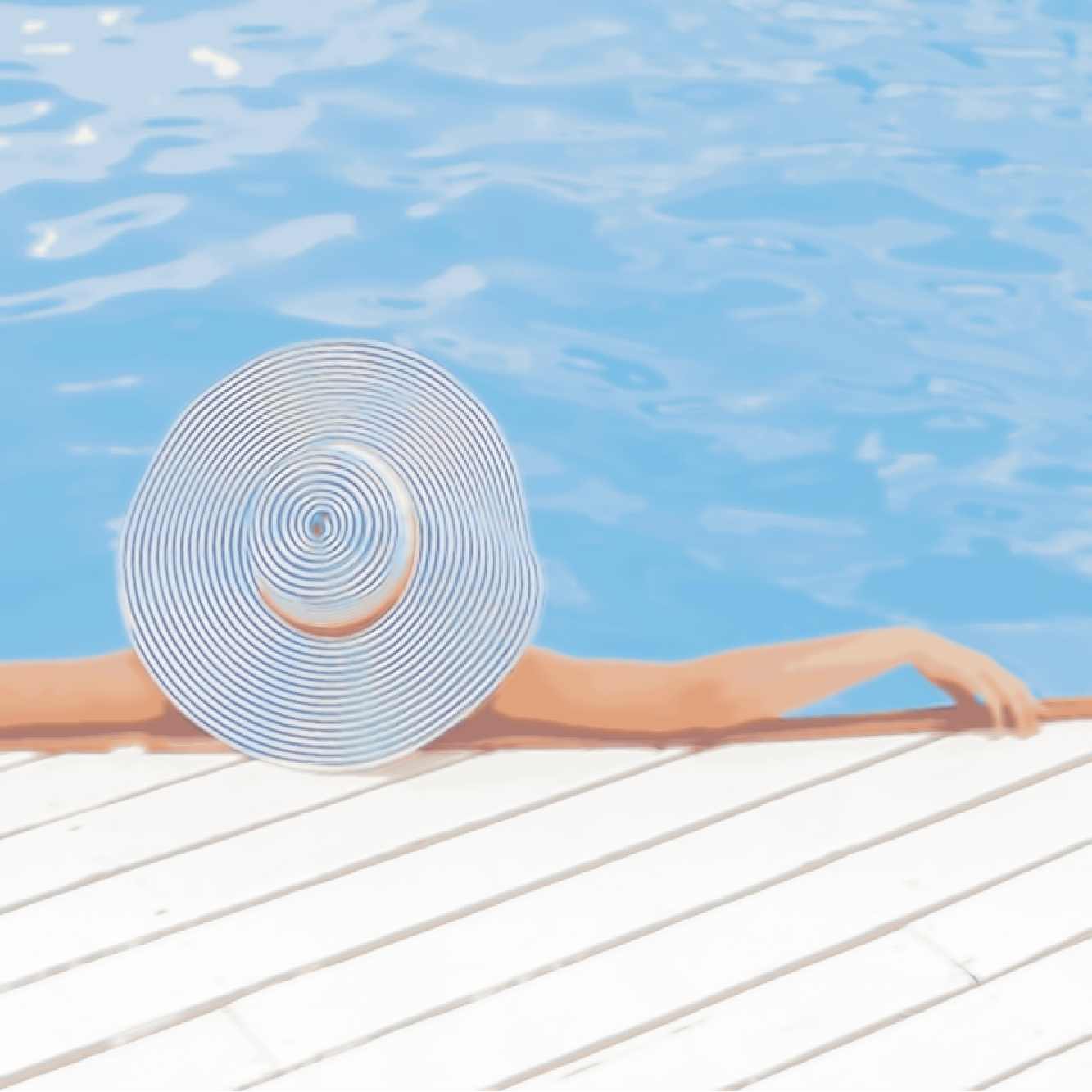} \\
      
       &
      \includegraphics[width=\mylength]{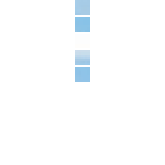} &
      \includegraphics[width=\mylength]{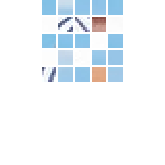} &
      \includegraphics[width=\mylength]{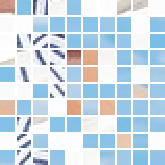}  \\ 
      
      &
      \includegraphics[width=\mylength]{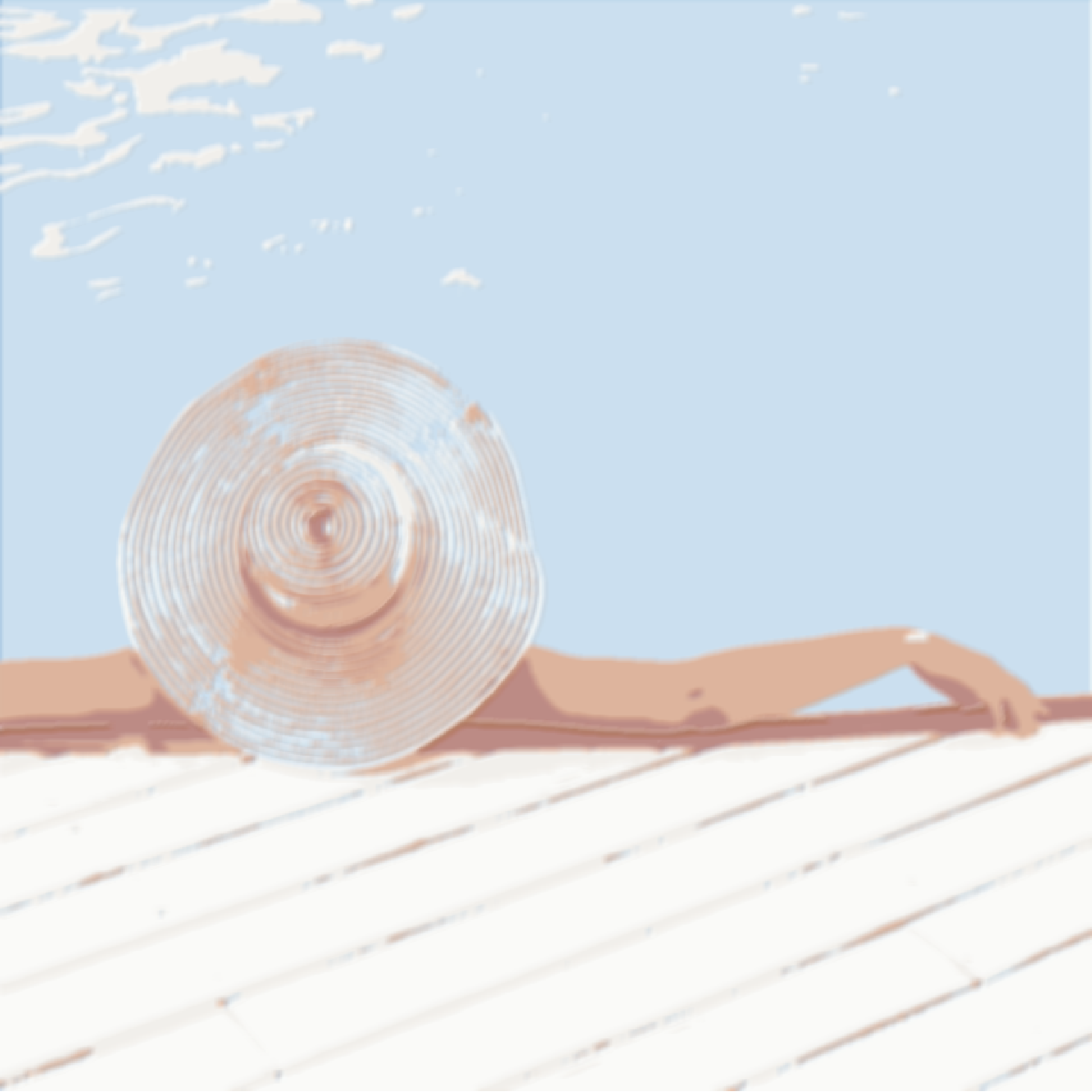} &
      \includegraphics[width=\mylength]{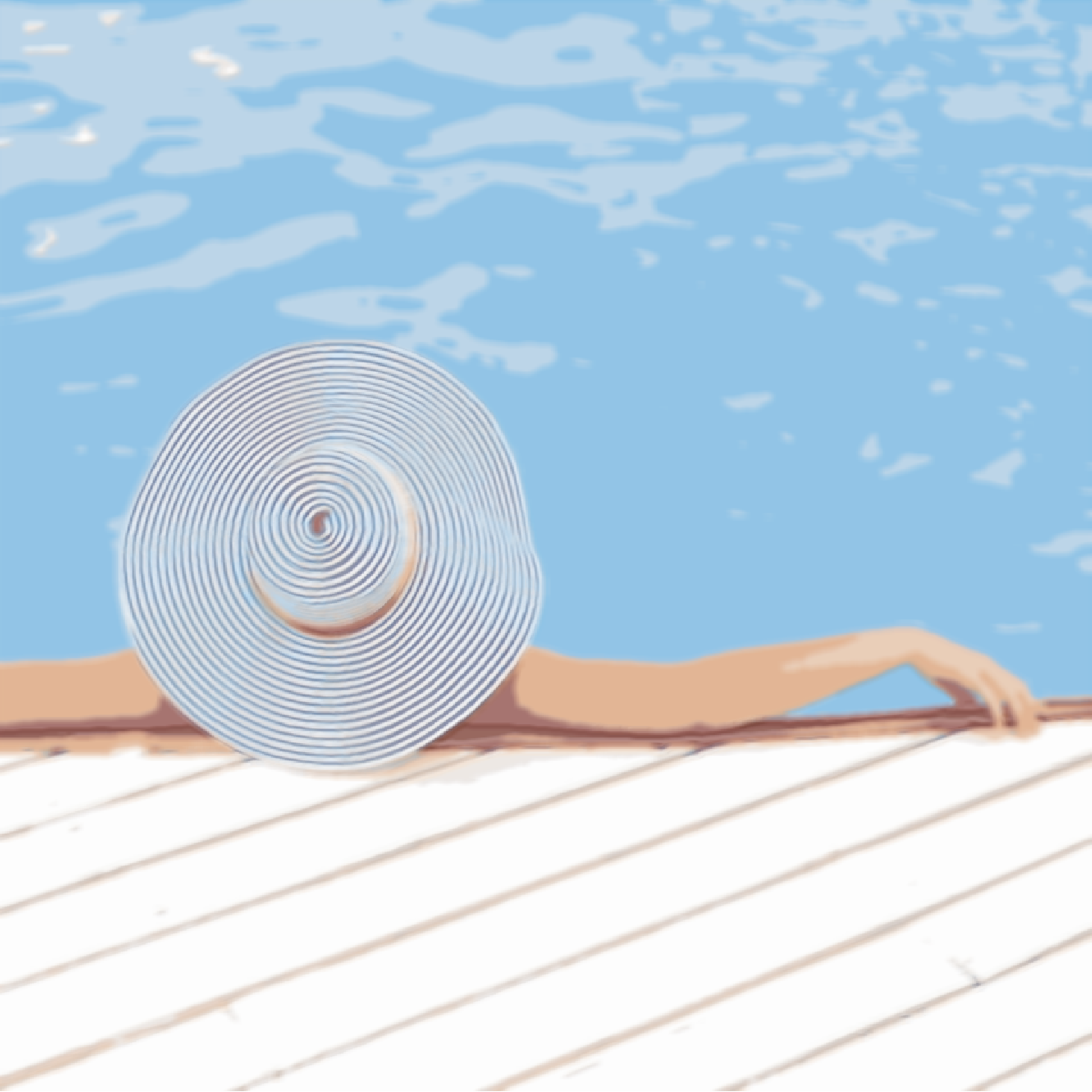} &
      \includegraphics[width=\mylength]{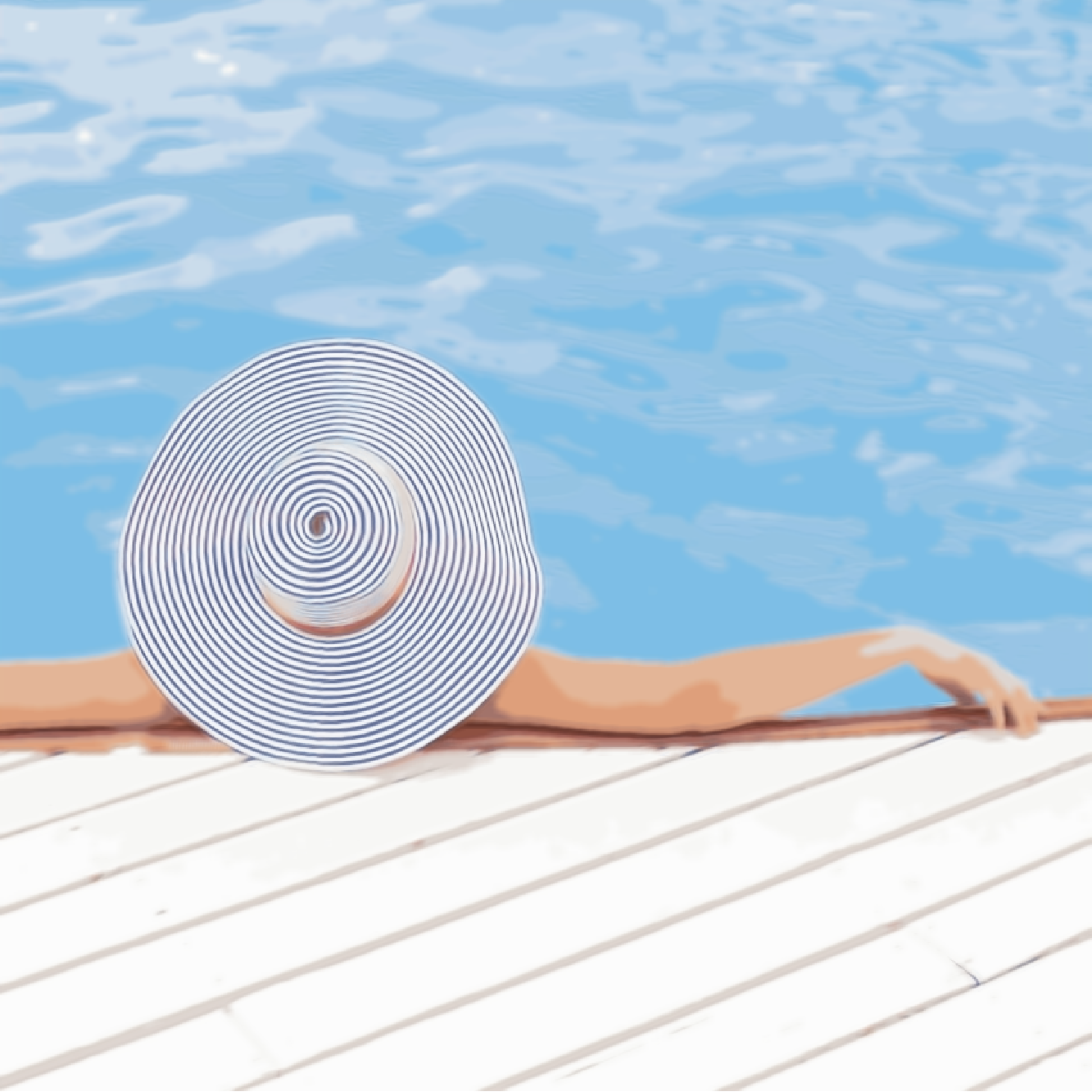}  \\
      
      &
      \includegraphics[width=\mylength]{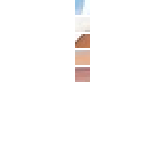} &
      \includegraphics[width=\mylength]{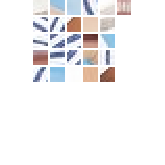} &
      \includegraphics[width=\mylength]{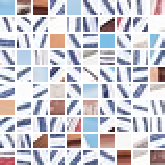}  \\
            Original & Card=$5$ & Card=$25$ & Card=$100$
    \end{tabular}
  \end{center}
  \caption{Image reconstruction: for each image, first two rows: original and reconstructions with uniformly sampled patches and below, the corresponding selected patches; second two rows: reconstructions with patches sampled according to a DPP and below, the corresponding selected patches.}
  \label{fig:reconstruction}
\end{figure}

\section{Discussion}

In this paper, we proposed a new sampling algorithm (Algorithm \ref{alg:sequential_thinning_dpp_bernoulli}) adapted to general determinantal point processes, which doesn't use the spectral decomposition of the kernel and which is exact. It proceeds in two phases. The first one samples a Bernoulli process whose distribution is adapted to the targeted DPP. It is a fast and efficient step that reduces the initial number of points of the ground set. We know that if $I-K$ is invertible, the expectation of the cardinality of the Bernoulli process is proportional to the expectation of the cardinality of the DPP. 
The second phase is a sequential sampling from the points selected in the first step. This phase is made possible by the explicit formulations of the general marginals and the pointwise conditional probabilities of any DPP from its kernel $K$. The sampling is sped up using updated Cholesky decompositions to compute the conditional probabilities. \textsc{Matlab} and Python implementations of the sequential thinning algorithm can be found online (\url{https://www.math-info.univ-paris5.fr/~claunay/exact_sampling.html}).

In terms of running times, we have detailed the cases for which this algorithm is competitive with the spectral algorithm, in particular when the size of the ground set is high and the expected cardinality of the DPP is modest. This framework is common in machine learning applications. 
Indeed, DPPs are an interesting solution to subsample a data set, initialize a segmentation algorithm or summarize an image, examples where the number of datapoints needs to be significantly reduced.

\appendix
\section{M\"{o}bius Inversion formula}
\label{app:mobius}

\begin{proposition}[M\"{o}bius inversion formula]
\label{prop:mobius_inversion_formula}
Let $V$ be a finite subset and $f$ and $g$ be two functions defined on the power set $\mathcal{P}(V)$ of subsets of $V$.
Then,
$$
\forall A\subset V,\quad f(A) = \sum_{B\subset A} (-1)^{|A\setminus B|} g(B)
\quad
\Longleftrightarrow
\quad
\forall A\subset V,\quad g(A) = \sum_{B\subset A} f(B),
$$
and
$$
\forall A\subset V,\quad f(A) = \sum_{B\supset A} (-1)^{|B\setminus A|} g(B)
\quad
\Longleftrightarrow
\quad
\forall A\subset V,\quad g(A) = \sum_{B\supset A} f(B).
$$
\end{proposition}

\begin{proof}
The first equivalence is proved e.g. in~\cite{Mumford_Desolneux_pattern_theory_2010}.
The second equivalence corresponds to the first applied to $\widetilde{f}(A) = f(\overline{A})$ and $\widetilde{g}(A) = g(\overline{A})$.
You will find more details on this matter in the book of Rota~\cite{Rota_Foundations_combinatorial_theoryI_1964}.\end{proof}

\section{Cholesky Decomposition Update}
\label{sec:cholesky_update}

To be efficient, the sequential algorithm relies on Cholesky decompositions that are updated step by step to save computations.
Let $M$ be a symmetric semi-definite matrix of the form $\displaystyle
M = 
\begin{pmatrix}
A & B
\\
B^t & C
\end{pmatrix}
$
where $A$ and $C$ are square matrices.
We suppose that the Cholesky decomposition $T_A$ of the matrix $A$ has already been computed and we want to compute the Cholesky decomposition $T_M$ of $M$.
Then, set 
$$
V = T_A^{-1} B \quad \text{ and } \quad
X = C - V^t V = C - B^t A^{-1} B
$$
the Schur complement of the block $A$ of the matrix $M$. Denote by $T_X$ the Cholesky decomposition of $X$.
Then, the Cholesky decomposition of $M$ is given by
$$
T_M = 
\begin{pmatrix}
T_A & 0
\\
V^t & T_X
\end{pmatrix}.
$$
Indeed,
$$
T_M T_M^t
= 
\begin{pmatrix}
T_A & 0
\\
V^t & T_X
\end{pmatrix}
\begin{pmatrix}
T_A^t & V
\\
0 & T_X^t
\end{pmatrix}
 =
\begin{pmatrix}
T_A T_A^t & T_A V 
\\
V^t T_A^t & V^t V + T_X T_X^t
\end{pmatrix}
=
\begin{pmatrix}
A & B
\\
B^t & C
\end{pmatrix}.
$$

\ack
Accepted for publication by the Applied Probability Trust (http://www.appliedprobability.org) in the Journal of Applied Probability JAP 57.4 (December 2020). This work was supported by grants from R\'egion Ile-de-France. We thank the reviewers for their valuable comments and suggestions that helped us to improve the paper.

\bibliographystyle{APT}      
\bibliography{biblio_dpp}

\end{document}